\documentclass{article}

\usepackage[english]{babel}
\usepackage{natbib}
\usepackage{amssymb}
\usepackage{amsfonts}
\usepackage[letterpaper,top=2cm,bottom=2cm,left=3cm,right=3cm,marginparwidth=1.75cm]{geometry}
\usepackage{amsmath, amsthm, amssymb}
\usepackage{graphicx}
\usepackage{xcolor}
\usepackage{hyperref}
\hypersetup{
    colorlinks=false,       
    citebordercolor=green,  
    linkbordercolor=black,  
    pdfborderstyle={/S/S/W 2} 
}

\newtheorem{assumption}{Assumption}
\usepackage{subcaption}
\usepackage{mathtools}
\usepackage{bbm}
\usepackage{enumitem}
\usepackage{booktabs}
\usepackage{tabularx}
\newcolumntype{Y}{>{\centering\arraybackslash}X}
\usepackage{wrapfig}
\graphicspath{{pictures/}}

\theoremstyle{plain}
\newtheorem{lemma}{Lemma}[section]

\usepackage{thmtools}
\usepackage{algorithm}
\usepackage{algpseudocode}

\DeclareMathOperator*{\argmin}{argmin}

\newtheorem{theorem}{Theorem}[section]
\newtheorem{proposition}[theorem]{Proposition}

\newtheorem{corollary}[theorem]{Corollary}

\newtheorem{remark}[theorem]{Remark}

\title{Spectrally Deconfounded Gradient Boosting}
\author{
  Andrea Nava\footnotemark[1] \footnotemark[2] \footnotemark[3]\\
    \and
 Peter Bühlmann\footnotemark[1] 
  \and
  Fabio Sigrist\footnotemark[1] 
}
\date{} 
\begin{document}

\maketitle

\begin{abstract}
Flexible machine-learning methods can be sensitive to hidden confounding: they may learn associations induced by unobserved confounders rather than stable signals. Spectral deconfounding mitigates this problem by shrinking high-variance directions of the covariate matrix that, under dense confounding, carry latent confounder information. Existing work has largely focused on linear models. We develop a nonlinear spectral deconfounding framework for gradient boosting. Our approach replaces the ordinary squared-error loss by a spectral loss, which alters the boosting dynamics by slowing down learning in confounding-aligned directions. We show that deconfounding is not achieved by the spectral loss alone, but by the interaction between spectral shrinkage and regularization, especially in terms of early stopping. Moreover, we provide a mixed-model interpretation that connects LAVA-type shrinkage to random-effects adjustment and yields an empirical-Bayes procedure for tuning the spectral loss. We also extend the method to general likelihoods and nonlinear confounding using Laplace approximations and kernel random effects. Across synthetic and real-world experiments, spectrally deconfounded boosting improves estimation of the target function under hidden confounding and is substantially more scalable than existing nonlinear spectral deconfounding baselines.

\vspace{0.5em}
\noindent \textbf{Keywords:} boosting, causal inference, confounding, random effects, mixed models

\end{abstract}

\footnotetext[1]{Seminar for Statistics, ETH Zürich, Zürich, Switzerland}
\footnotetext[2]{IFZ, Lucerne University of Applied Sciences and Arts, Rotkreuz, Switzerland}
\footnotetext[3]{Corresponding author: andrea.nava@stat.math.ethz.ch}

\section{Introduction}
Causal inference from observational data is fundamentally challenging in the presence of unobserved confounding. When latent variables influence both covariates and an outcome, naive estimation of the effect of covariates on the outcome is generally biased. Beyond causal estimation, hidden confounding can also harm predictive stability: flexible learners may exploit spurious associations that are predictive on the observed distribution but fail to generalize to new environments, thereby degrading robustness under distribution shift \citep{e63a2e52-1bf2-3782-baa4-27c7f1592402}. Gradient boosting \citep{friedman2000additive, friedman2001greedy} is one of the most accurate methods for prediction on tabular data \citep{grinsztajn2022tree}, but precisely because of its flexibility it is vulnerable to hidden confounding: when confounding-driven patterns are predictive in sample, boosting can fit them and thereby lose robustness under changes in the confounding structure. In this paper, we study how to improve the robustness of gradient boosting to hidden confounding. Since latent confounding cannot be resolved from observational data without additional structure, this requires explicit assumptions on how confounding enters the data-generating process.

We consider the dense-confounding regime. The key idea is that when latent confounders affect many observed covariates simultaneously, their influence leaves a detectable imprint in the covariance geometry of the design matrix, typically concentrating in the directions corresponding to the largest singular values of the design matrix. This intuition already appears in genome-wide association studies (GWAS) through principal component adjustment, where a small number of top principal components are included to control for population structure \citep{price2006principal, novembre2008genes}, and in linear mixed models, which introduce random effects with covariance matrices derived from the design matrix and thereby adjust for confounding without explicitly selecting the number of principal components \citep{zhang2010mixed, sul2018population}.

Spectral deconfounding \citep{cevid2020spectral} takes this route. Rather than estimating latent confounders, it applies a spectral transformation that shrinks directions aligned with high variance and, under strong and dense confounding, with the confounding structure. 
The underlying dense-confounding assumption is closely related to the pervasive-factor assumption studied in high-dimensional factor models \citep{bai2003inferential, fan2008high, lam2012factor}. However, while factor-model methods aim to estimate the latent factors and their loadings, spectral deconfounding exploits the resulting covariance structure to attenuate confounding effects. Combined with regularization, this yields recovery of the sparse signal without requiring explicit confounder identification \citep{cevid2020spectral,guo2022doubly}.
\citet{10.1145/3711116} provide rigorous mathematical analysis for nonlinear additive models while \citet{Ulmer03042026} propose a methodology for Random Forests. Our goal is to bring this logic to nonlinear function learning with boosting.
We enable deconfounding by specifying a spectral loss for the boosting objective function.
Moreover, we show that the spectral loss alone does not remove confounding bias: if the learner is allowed to fit indefinitely, it can still absorb confounding-driven structure. Robustness instead arises from the interaction between the transformed loss and regularization. 

Our central contribution is to show that spectral deconfounding can be understood as an optimization-path modification for boosting. Rather than removing hidden confounding by a preprocessing transformation alone, the spectral loss changes the gradient dynamics of boosting: directions associated with large singular values of the design, which are expected to contain dense confounding signal, are learned more slowly. Deconfounding therefore arises from the interaction between spectral shrinkage and regularization, in particular in terms of early stopping. This perspective allows us to extend spectral deconfounding from linear estimators to nonlinear gradient-boosted trees. 

More specifically, we make the following contributions. First, we introduce a spectral-loss formulation for gradient boosting under dense hidden confounding. This enables spectral deconfounding for nonlinear models when using nonlinear base learners such as trees. Second, we theoretically analyze boosting with linear base learners and show that spectral weighting creates an intermediate stopping regime in which signal directions are learned while confounding-aligned directions remain suppressed. Third, we derive the spectral loss from a mixed-model formulation, connecting LAVA-type spectral shrinkage \citep{chernozhukov2017lava} to random-effects adjustment and yielding an empirical-Bayes procedure for tuning the transformation. Fourth, we extend the framework to non-Gaussian likelihoods and nonlinear confounding, and provide a scalable software implementation. To the best of our knowledge, this is the first spectral deconfounding approach for generalized nonlinear models and nonlinear confounding.

\paragraph{Organization of the paper.} The rest of the paper is organized as follows.
Section~\ref{sec:related_work} reviews linear spectral deconfounding and establishes a connection to mixed models. Section~\ref{sec:opt_view} develops an optimization view that highlights the role of regularization. Section~\ref{sec:non-linear} introduces the nonlinear extension, derives the spectral loss from a random-effects formulation, and presents the resulting boosting algorithm. Section~\ref{sec:generalized_lava_classification} extends the method to non-Gaussian likelihoods and to nonlinear confounding. Section~\ref{sec:experiments} reports simulation studies illustrating the benefit of spectral deconfounding and the impact of early stopping and iteration control under confounding. Section~\ref{sec:real_data} presents a real-world data case study. Section~\ref{sec:conclusion} concludes with limitations and future directions.

\paragraph{Notation.}
For a matrix \(A \in \mathbb{R}^{m \times n}\), we write \(d_i(A)\) for its \(i\)-th largest (non-zero) singular value, and use \(d_{\max}(A)\) and \(d_{\min}(A)\) for its largest and smallest (non-zero) singular value, respectively. 
We write \(\operatorname{col}(A)\) for the column space of \(A\), and \(I_m\) for the \(m \times m\) identity matrix. For random vectors \(Z\) and \(W\), \(\operatorname{Cov}(Z,W)\) denotes their cross-covariance matrix, and \(\operatorname{Cov}(Z) := \operatorname{Cov}(Z,Z)\).
We write \(\|\cdot\|_2\) for the Euclidean norm of a vector, the corresponding operator norm of a matrix, and $\langle \cdot,\cdot \rangle$ for the Euclidean inner product. Finally, for two positive sequences \(a_p\) and \(b_p\), we write \(a_p \gtrsim b_p\) if there exists a constant \(c>0\), independent of \(p\), such that \(a_p \geq c\, b_p\) for all sufficiently large \(p\).

\section{Background on spectral deconfounding for linear models}
In this section, we review existing work on linear spectral deconfounding and highlight the connection between spectral methods and random effects adjustments. 
\label{sec:related_work}
\subsection{Confounding in linear models}
\label{sub:linear}

Consider the structural equation model for $n$ i.i.d. observations with unobserved confounding:
\begin{align}
    X &= H\Gamma + E, \label{eq:sem_x}\\
    Y &= X\beta_0 + H\delta + \varepsilon, \label{eq:sem_y}
\end{align}
where $H \in \mathbb{R}^{n \times q}$ are unobserved confounders, $\Gamma \in \mathbb{R}^{q \times p}$ captures their effects on observed variables $X \in \mathbb{R}^{n \times p}$, $\beta_0 \in \mathbb{R}^p$ is the direct causal effect of $X$ on $Y$, $\delta \in \mathbb{R}^q$ is the confounding effect on the outcome $Y \in \mathbb{R}^n$, $E$ is the unconfounded design independent of $H$, and $\varepsilon$ is outcome noise independent of $X$ and $H$. 

Bias arises because confounders induce correlation between the covariates $X$ and the unobserved component \(H\delta\) affecting \(Y\). To make this precise, write \(x_i \in \mathbb{R}^p\) and \(h_i \in \mathbb{R}^q\) for generic row-level random vectors corresponding to the rows of \(X\) and \(H\). We work with centered covariates, so that $\mathbb{E}[x_i]=0$. Define the population linear projection coefficient
\[
    b_0 \in \arg\min_{a \in \mathbb{R}^p}
    \mathbb{E}\left[\left(h_i^\top\delta - x_i^\top a\right)^2\right].
\]
The projection residual
$
    r_i := h_i^\top\delta - x_i^\top b_0
$
satisfies \(\operatorname{Cov}(x_i,r_i)=0\). Stacking these residuals over
the \(n\) observations gives
\[
    H\delta = Xb_0 + r,
\]
where \(r = (r_1,\ldots,r_n)^\top\). Substituting into the structural equation
yields
\[
    Y = X(\beta_0+b_0) + r + \varepsilon.
\]
The vector $b_0$ thus captures the bias introduced by confounding, as it depends on the correlation between $H\delta$ and $X$, while 
the error vector $(H\delta - Xb_0)$ is uncorrelated with the covariates in the population projection case.

\subsection{Spectral deconfounding}
\label{subsec:spec_deconfounding}
Spectral deconfounding \citep{cevid2020spectral} methods apply a linear transformation \(F \in \mathbb{R}^{n \times n}\) to both the covariates and the outcome,
\begin{equation}
\label{eq:pre_processing}
\tilde{X} := FX, \qquad \tilde{Y} := FY,
\end{equation}
in order to reduce or remove the alignment between the covariates and the bias vector \(b_0\). More precisely, let the singular value decomposition of \(X\) be \(X = UDV^\top\), where \(D \in \mathbb{R}^{k \times k}\) is diagonal with entries \(d_1,\dots,d_k\), \(U \in \mathbb{R}^{n\times k}\) and \(V \in \mathbb{R}^{p\times k}\) have orthonormal columns, and \(k = \mathrm{rank}(X) \leq \text{min}(n,p)\). A spectral transformation takes the form
\begin{align}
    F = U\tilde{D}U^\top
    = U\,\mathrm{diag}(\tilde d_1/d_1,\dots,\tilde d_k/d_k)\,U^\top,
    \label{eq:spectral_transform}
\end{align}
where the transformed singular values \(\tilde d_i \geq 0\) are chosen to attenuate directions associated with confounding while preserving the causal signal. In other words, the goal is that the confounding component \(Xb_0\), transformed into \(\tilde{X}b_0\), is substantially attenuated, while the causal component \(\tilde{X}\beta_0\) retains most of its variance. 
In the original linear framework of \citet{cevid2020spectral}, the spectral transformation is combined with a regularized estimator, typically the lasso, fitted to the transformed data. The underlying assumption is that \(\beta_0\) is sparse, while the residual confounding contribution $\|\tilde{X}b_0\|_2/\sqrt n$ is sufficiently attenuated so that an \(\ell_1\)-regularization will not select signal from the term $\tilde{X}b_0$.

For linear squared-error objectives, the preprocessing view in~\eqref{eq:pre_processing} is equivalent to replacing the ordinary residual norm by the transformed norm
\[
\|F(Y-X\beta)\|_2^2.
\]
We later refer to such transformed squared-error objectives as spectral losses. This loss-based view is made explicit in Section~\ref{sec:mm}, where it arises naturally from a mixed-model likelihood, and is later used to extend spectral deconfounding to nonlinear predictors \(f(X)\), for which a preprocessing interpretation of $X$ is no longer available.

\subsubsection{Some intuition}
To see the spectral implication of confounding, let $\Sigma_H := \operatorname{Cov}(h_i)$, $\Sigma_E := \operatorname{Cov}(e_i)$, and $\Sigma_X := \operatorname{Cov}(x_i)$ denote the population covariance matrices of one observation. Assume $e_i$ is independent of $h_i$ and without loss of generality that $\Sigma_{H} = I_q$, then the covariance of $x_i$ is $\Sigma_{X} = \Gamma^\top \Gamma + \Sigma_{E}$.
When the leading eigenvalues of $\Gamma^\top\Gamma$ are sufficiently
separated from those of $\Sigma_E$, the leading eigenspace of the population
covariance $\Sigma_X$ is close to the loading subspace $\operatorname{col}(\Gamma^\top)$. For the realized design matrix $X$, the
empirical covariance $X^\top X/n$ concentrates around $\Sigma_X$ in large
samples. Since the right singular vectors of $X$ are the eigenvectors of
$X^\top X$, the leading right singular directions of $X$ approximately
recover the same subspace \citep{bai2003inferential}.

At the same time, the population bias vector satisfies
\[
b_0
=
\Sigma_{X}^{-1}\operatorname{Cov}(x_i, h_i^\top\delta)
=
\Sigma_{X}^{-1}\Gamma^\top \Sigma_{H}\delta
=
\Sigma_{X}^{-1}\Gamma^\top \delta,
\]
so, under the eigenvalue-separation condition described above, \(b_0\) is approximately 
aligned with the loading subspace \(\operatorname{col}(\Gamma^\top)\), 
and hence with the leading right singular directions of \(X\). This implies that the sample-space confounding component \(Xb_0\) has large $\ell_2$-norm and is concentrated in the leading left singular directions of \(X\). This consideration motivates the following choices of spectral transformations used in the literature.

\paragraph{Trim Transform}
The Trim transform, proposed by \citet{cevid2020spectral}, caps the leading singular values at a common level without requiring hyperparameter tuning:
\begin{equation}
\label{eq:trim}
\tilde d_i =
\begin{cases}
\mathrm{median}(\{d_j\}_{j=1}^k), & \text{if } i \leq \lfloor k/2 \rfloor,\\
d_i, & \text{otherwise},
\end{cases}~~~~~~i=1,\dots,k.
\end{equation}
It is also the transformation used in spectrally deconfounded random forests \citep{Ulmer03042026}, making it a natural benchmark in our experiments.

\paragraph{LAVA Transform}

The LAVA method \citep{chernozhukov2017lava}, originally developed to deal with so-called ``sparse+dense" models, combines a sparse component \(\beta\) with a dense component \(b\) through the objective
\begin{align}
    (\hat{\beta}, \hat{b}) = \argmin_{\beta,b}
    \frac{1}{n}\|Y - X(\beta + b)\|_2^2
    + \lambda_2\|b\|_2^2
    + \lambda_1\|\beta\|_1.
    \label{eq:lava_obj}
\end{align}
Profiling out \(b\) from \eqref{eq:lava_obj} yields a lasso problem after applying the spectral transformation
\begin{align}
    F = \Bigl(I_n - X(X^\top X + n\lambda_2 I_p)^{-1}X^\top\Bigr)^{1/2}.
    \label{eq:lava_transform}
\end{align}
The corresponding shrinkage in terms of singular values of the design matrix is
\begin{align}
    \tilde d_i = \sqrt{\frac{n\lambda_2\, d_i^2}{n\lambda_2 + d_i^2}},~~i=1,\dots,k.
    \label{eq:lava_shrinkage}
\end{align}
Unlike principal component adjustment \citep{price2006principal}, a LAVA transform shrinks all directions continuously rather than removing some of them entirely. This is the transformation we build on later, since its \(\ell_2\)-penalized component admits a random effects interpretation and yields a principled route to tuning \(\lambda_2\) (see Section \ref{sec:mm}).

\subsection{The role of dense confounding}

The key structural assumption behind spectral deconfounding is \emph{dense confounding}. In the linear structural equation model \(X = H\Gamma + E\), this is formalized by requiring that the confounding loadings are sufficiently spread out across coordinates, for example through the condition \(d_{\min}(\Gamma) \gtrsim \sqrt{p}\) as \(p \to \infty\) \citep{cevid2020spectral}. Intuitively, this means that each latent confounder affects many observed covariates, so that the resulting bias is distributed broadly across covariates rather than concentrated on a small subset. 

Dense confounding, together with additional regularity conditions, is used in the spectral deconfounding literature to show that the coefficient-level bias $b_0$ vanishes asymptotically in \(\ell_2\)-norm \citep[Lemma~6]{cevid2020spectral}. 
This helps explain why ordinary least squares can become asymptotically unbiased in such regimes and also underlies critiques of latent-confounder recovery based approaches, such as that of \citet{grimmer2023naive}. However, for regularized estimators, and especially in high-dimensional settings, the relevant quantity is not only \(\|b_0\|_2\), but also the geometry of the bias \(Xb_0\) in sample space \citep[Theorem~2]{cevid2020spectral}. Even when \(\|b_0\|_2\) is small, $\|Xb_0\|_2$ can remain non-negligible if \(b_0\) is aligned with high-variance directions of \(X\), and regularized procedures may still be attracted to fitting this component. This is particularly natural for estimators that favor high-variance directions, such as ridge-type procedures, gradient descent with early stopping, or sparsity-based methods applied to anisotropic designs.

The role of spectral deconfounding is therefore not to eliminate the coefficient bias \(b_0\) itself, but to control the fitted confounding term by reshaping the spectrum of the design. In other words, for a fixed design dimension \(p\), a spectral transformation does not alter the population bias vector \(b_0\), but it does change the geometry of the fitted bias in sample space through \(Xb_0 \mapsto FXb_0\) to attenuate its alignment with directions that are easy for a regularized learner to fit. This is precisely the mechanism we exploit in the gradient boosting setting. 

\subsection{Mixed models and random-effects adjustment}
\label{sec:mm}
Mixed-model approaches in genome-wide association studies (GWAS) \citep{zhang2010mixed, sul2018population} are closely related to spectral deconfounding. To account for population stratification and other latent structure, these methods introduce a Gaussian random effect, leading to the linear mixed model
\begin{align}
Y
=
X\beta
+
Xb
+
\epsilon,
\qquad
b \sim \mathcal{N}(0,\sigma_r^2I_{p}),
\qquad
\epsilon \sim \mathcal{N}(0,\sigma_e^2 I_n).
\label{eq:mixed_model_full}
\end{align}
This induces a marginal covariance of the form $\sigma_r^2XX^\top + \sigma_e^2I_{n}$, where $XX^\top$ plays the role of a so-called kinship or relatedness matrix.
The random effect is intended to capture variation induced by latent structure in \(X\), such as population stratification or other dense confounding effects. Since \(XX^\top = UD^2U^\top\), the random effect acts along the principal component directions of \(X\), with shrinkage modulated continuously by the spectrum \(D\). In this sense, mixed models can be viewed as a continuous analogue of principal component adjustment \citep{price2006principal}: instead of selecting a finite number of components, they shrink all principal directions using a Gaussian prior.

Under the Gaussian likelihood, integrating out the random effect yields
\[
Y\mid X
\sim
\mathcal{N}\!\left(X\beta,\,\Sigma\right),
\qquad
\Sigma = \sigma_r^2 XX^\top + \sigma_e^2 I_n.
\]
The corresponding maximum likelihood estimator is the generalized least squares estimator
$
\hat{\beta}
=
(X^\top \Sigma^{-1} X)^{-1}X^\top \Sigma^{-1}Y$, 
which minimizes the transformed squared loss
\begin{equation}\label{spec_loss}
    L_F(\beta)
    :=
    \frac{1}{2}\|F(Y-X\beta)\|_2^2,
    \qquad F=\Sigma^{-1/2}.
\end{equation}
Equivalently $\hat{\beta}$ is obtained by ordinary least squares on the whitened data \(FX\) and \(FY\).
As mentioned in Section~\ref{subsec:spec_deconfounding}, we call a loss of this form a spectral loss when it is a squared-error loss transformed by a spectral transformation \(F\) that acts diagonally in the left singular-vector basis of \(X\).
In the mixed-model formulation above, the whitening matrix \(F=\Sigma^{-1/2}\) provides one such spectral transformation.
Indeed, applying the Sherman--Morrison--Woodbury identity  to $\Sigma^{-1}$ gives
\begin{equation}
\label{eq:precision_smw}
\Sigma^{-1}
=
(\sigma_r^2 XX^\top + \sigma_e^2 I_n)^{-1}
=
\frac{1}{\sigma_e^2}\Bigl(
I_n
-
X(X^\top X + \tfrac{\sigma_e^2}{\sigma_r^2}I_p)^{-1}X^\top
\Bigr).
\end{equation}
Thus, up to the irrelevant global factor \(1/\sigma_e^2\), the whitening operator induces exactly the same spectral transformation as the LAVA transform in \eqref{eq:lava_transform}, with the identification
$
n\lambda_2 = \sigma_e^2/\sigma_r^2.
$
Equivalently, the transformed singular values satisfy
\begin{equation}
\tilde d_i
=
\frac{d_i}{\sqrt{1 + (\sigma_r^2/\sigma_e^2)d_i^2}}
=
\sqrt{\frac{n\lambda_2\,d_i^2}{n\lambda_2 + d_i^2}},~~i=1,\dots,k.
\end{equation}
This matches \eqref{eq:lava_shrinkage}. The discussion above can be summarized in the following proposition.

\begin{proposition}[Mixed-model interpretation of LAVA]
\label{prop:lava_mixed_model}
Consider the linear mixed model in \eqref{eq:mixed_model_full}, with
\(\Sigma=\operatorname{Cov}(Y|X)=\sigma_r^2XX^\top+\sigma_e^2I_n\). Then the normalized whitening transformation $
\bar F=\sigma_e\Sigma^{-1/2}
$
coincides with the LAVA spectral transformation in \eqref{eq:lava_transform}, with the identification
$
n\lambda_2= \sigma_e^2/\sigma_r^2.
$
Equivalently,
$$
\sigma_e\Sigma^{-1/2}
=
\left( I_n-X(X^\top X+n\lambda_2 I_p)^{-1}X^\top \right)^{1/2} .
$$
Thus, the LAVA shrinkage parameter \(\lambda_2\) can be interpreted as a variance-component ratio.
\end{proposition} 
This observation provides, to the best of our knowledge, a novel explicit link between the LAVA transform used in spectral deconfounding and random-effects adjustment in mixed models. 

\section{Asymptotic Unbiasedness of Spectrally Deconfounded Gradient Boosting with Linear Base Learners}
\label{sec:opt_view}

In this section, we study spectral deconfounding from an optimization perspective. Our goal is to understand how spectral losses interact with boosting and, in particular, how they affect the learning of
confounding-aligned directions. The key idea is that boosting \citep{freund1996experiments, 10.1214/aos/1024691079} can be viewed as gradient descent in function space \citep{friedman2000additive, friedman2001greedy, buhlmann2003boosting}. 
We show that a spectral transformation of the loss induces direction-dependent learning rates.
Combined with early stopping, this leads to asymptotically unbiased prediction even in the presence
of confounding.

We consider a sequence of problems indexed by $p$, with $p \to \infty$ and $n=n_p \rightarrow \infty$, and suppress the dependence of $n$ on $p$ for notational simplicity.  
Let $X\in\mathbb{R}^{n\times p}$ be a
realized design matrix of rank $k$, with singular value decomposition $X=UDV^\top$, where
$U=(u_1,\dots,u_k)\in\mathbb{R}^{n\times k}$ has orthonormal columns.
We work with the model
\[
Y = f_0 + g_0 + \varepsilon,
\qquad
f_0 = X\beta_0 \in \mathbb{R}^{n},
\qquad
g_0 = Xb_0 \in \mathbb{R}^{n},
\]
where $f_0$ represents the signal, $g_0$ represents the confounding contribution, and $\varepsilon \in \mathbb{R}^n$ is additive noise. We condition on $X$ throughout.

We use squared losses of the form
\begin{equation}
L_W(f)
=
\frac{1}{2}\|W^{1/2}(Y-f)\|_2^2,
\end{equation}
where $W\in \mathbb{R}^{n \times n}$ is a symmetric positive semidefinite spectral weight matrix. The ordinary squared loss
corresponds to $
W^{\mathrm{or}} = I_n,
$
whereas the spectral loss corresponds to
\[
W^{\mathrm{sp}} = F^\top F
=
U\operatorname{diag}(w_1,\dots,w_k)U^\top,
\qquad w_i \geq 0.
\]
This is equivalent to the transformed-singular-value parametrization used in Section~\ref{subsec:spec_deconfounding}: if the spectral transformation $F$ replaces \(d_i\) by any \(\tilde d_i \geq 0\), then $$
w_i
=
(\tilde d_i / d_i)^2.$$ 
Thus, specifying the transformed singular values \(\tilde d_i\) or the weights \(w_i\) is equivalent, with
\(\tilde d_i=d_i\sqrt{w_i}\). In this section, we use the weight parametrization \(w_i\), since these are the quantities that enter the boosting recursion below.
For example, the rescaled random-effects transformation $F
=
\sigma_e(\sigma_r^2 XX^\top+\sigma_e^2 I_n)^{-1/2}
$ induces

$$
w_i
=
\frac{\sigma_e^2}{\sigma_r^2 d_i^2+\sigma_e^2},
\qquad i=1,\ldots,k.
$$
Consequently, directions corresponding to large singular values of \(X\) receive smaller weights.
This weighting enters the boosting update through the gradient:
the negative gradient of \(L_W\) is
\[
-\nabla_f L_W(f)=W(Y-f).
\]

We consider boosting with a linear base learner. Since the ordinary least squares (OLS) fit of any pseudo-response
\(z\in\mathbb{R}^n\) on \(X\) is its orthogonal projection onto \(\operatorname{col}(X)\), the base
learner is given by
\[
P_X z,
\qquad
P_X := X(X^\top X)^+X^\top = UU^\top.
\]
We then study boosting with this base learner under the ordinary squared loss and under the
spectral loss introduced above.
Fixing a learning rate \(\nu\in(0,1)\) and starting from \(\hat f_0=0\), boosting under the weighted loss
\(L_W\) is defined recursively by
\[
\hat f_{m+1}
=
\hat f_m+\nu P_XW(Y-\hat f_m),
\qquad m\ge 0.
\]
The ordinary $L_2$-boosting recursion is obtained by taking \(W=W^{\mathrm{or}}=I_n\), while the spectral
boosting recursion is obtained by taking \(W=W^{\mathrm{sp}}=F^\top F\). We denote the two resulting
iterates by \(\hat f_m^{\mathrm{or}}\) and \(\hat f_m^{\mathrm{sp}}\), respectively. The advantage of this formulation is that both \(P_X\) and \(W^{\mathrm{sp}}\) diagonalize in the
basis \((u_i)_{i=1}^k\). To analyze the recursion, write
\[
f_0=\sum_{i=1}^k f_i u_i,
\qquad
g_0=\sum_{i=1}^k g_i u_i,
\]
where $
f_i:=\langle f_0,u_i\rangle,
g_i:=\langle g_0,u_i\rangle.$
Thus \(f_i\) and \(g_i\) are the coordinates of the signal and fitted confounding term along the left
singular directions of \(X\).

The conditional bias of the boosting iterates evolves now coordinatewise as
\[
\mathbb{E}[\hat f_m\mid X]-f_0
=
-\sum_{i=1}^k (1-\nu w_i)^m f_i u_i
+
\sum_{i=1}^k \bigl(1-(1-\nu w_i)^m\bigr) g_i u_i,
\]
where \(w_i=1\) in the ordinary case and \(w_i \geq 0\) are the spectral weights in the transformed
case (see Appendix \ref{app:proof_exact_bias} for a derivation). The key insight is that each direction \(u_i\) is learned at rate \(\nu w_i\): in ordinary boosting all directions are learned at the same rate, while spectral deconfounded boosting slows down directions with small weights.

To state our main theorem, we impose the following assumptions.

\begin{assumption}[Reduced model and bounded energy]
\label{ass:reduced_model_energy_main}
The noise satisfies $\mathbb{E}[\varepsilon\mid X]=0$, and
\[
\frac{\|f_0\|_2}{\sqrt n} = O(1),
\qquad
\frac{\|g_0\|_2}{\sqrt n} = O(1).
\]
\end{assumption}

\begin{assumption}[Spectral separation]
\label{ass:empirical_separation_main}
Fix $q\in\mathbb{N}$ with \(q<k\) for all sufficiently large \(p\), and define the empirical top-$q$ subspace projector
\[
\hat P_H := \sum_{i=1}^q u_i u_i^\top.
\]
The confounding term is concentrated in the top spectral subspace, while the signal is not:
\[
\frac{\|(I-\hat P_H)g_0\|_2}{\sqrt n} \to 0,
\qquad
\frac{\|\hat P_H f_0\|_2}{\sqrt n} \to 0.
\]
\end{assumption}

\begin{assumption}[Weight separation]
\label{ass:weight_separation_main}
Let
\[
w_H := \max_{1\le i\le q} w_i,
\qquad
w_M := \min_{q<i\le k} w_i.
\]
Assume \(0\le w_i\le 1\) for all \(i=1,\ldots,k\), and
\[
 (a) \quad w_H \to 0,
\qquad
(b) \quad \liminf_{p\to\infty} w_M \ge c > 0.
\]
\end{assumption}

The assumptions are discussed below after the theoretical statements. The following theorem shows that spectral weighting creates a window of stopping times for which the
signal is learned while the confounding term is not.

\begin{theorem}[Asymptotic unbiasedness]
\label{thm:pathwise_unbiasedness_main}
Suppose Assumptions~\ref{ass:reduced_model_energy_main}--\ref{ass:weight_separation_main} hold.
Let the boosting iterations $m_p$ satisfy
\[
(i) \quad m_p \to \infty,
\qquad
(ii) \quad m_p \nu w_M \to \infty,
\qquad
(iii) \quad m_p \nu w_H \to 0.
\]
For fixed $\nu$, condition (ii) is implied by Assumption 3(b), while condition (iii) is implied by Assumption 3(a) for sufficiently slowly growing number of iterations.
Then
\[
\frac{1}{\sqrt n}
\left\|
\mathbb{E}[\hat f^{\mathrm{sp}}_{m_p}\mid X] - f_0
\right\|_2
\to 0.
\]
\end{theorem}

\begin{corollary}[Failure of ordinary $L_2$-boosting]
\label{cor:plain_failure}
If $\liminf \|g_0\|_2/\sqrt n > 0$, then for any $m_p \to \infty$,
\[
\frac{1}{\sqrt n}
\left\|
\mathbb{E}[\hat f^{\mathrm{or}}_{m_p}\mid X] - f_0
\right\|_2
\not\to 0.
\]
\end{corollary}

\paragraph{Discussion.}
The assumptions in Theorem~\ref{thm:pathwise_unbiasedness_main} are intended to abstract the behavior of the dense-confounding structural equation model defined by \eqref{eq:sem_x} and \eqref{eq:sem_y}.
Their role is to separate three ingredients: a reduced sample-space representation of the bias,
spectral concentration of this bias, and algorithmic downweighting of the corresponding directions.

Assumption~\ref{ass:reduced_model_energy_main} follows from the usual projection argument in
the dense factor model.
Writing
\[
H\delta = Xb_0 + r, \qquad b_0 = \Sigma_X^{-1}\Gamma^\top\delta,
\]
the residual \(r=H\delta-Xb_0\) is the part of the confounding term orthogonal to the linear span of
\(X\). Under dense confounding, for instance when \(d_q(\Gamma)\gtrsim \sqrt p\), this residual
has vanishing normalized empirical norm, \(\|r\|_2/\sqrt n=o_p(1)\); see Appendix~\ref{app:factor_model_justification}.
If \((H,E)\) are jointly Gaussian, then this orthogonal projection residual is independent of \(X\),
and hence satisfies \(E[r\mid X]=0\). In that case, the reduced model
\[
Y = X\beta_0 + Xb_0 + \varepsilon'
\]
has the exact conditional-mean structure required in Assumption~\ref{ass:reduced_model_energy_main}.
Without joint Gaussianity, the same reduced model is valid up to an \(o_p(1)\) perturbation in
normalized empirical $\ell_2$ norm, because the orthogonal component \(r\) itself vanishes asymptotically.

Assumption~\ref{ass:empirical_separation_main} is the corresponding geometric condition in sample
space. Its bias part is justified by standard strong-factor arguments. Under the factor model
\(X=H\Gamma+E\), the factor component \(H\Gamma\) generates \(q\) leading singular directions,
and the empirical top-\(q\) principal subspace of \(X\) consistently estimates the factor space
spanned by \(H\), up to rotation, under standard factor-model conditions \citep{bai2003inferential}. Equivalently,
using perturbation results such as Davis--Kahan or Wedin-type sin-theta bounds \citep{yu2015useful}, the projector onto
the leading left singular vectors of \(X\) converges to the projector onto \(\operatorname{col}(H)\).
Since the fitted confounding component satisfies \(Xb_0=H\delta-r\) with
\(\|r\|_2/\sqrt n=o_p(1)\), the bias term \(Xb_0\) is asymptotically contained in this leading empirical
spectral subspace. The second part of Assumption~\ref{ass:empirical_separation_main}, namely that
the direct signal \(X\beta_0\) has negligible projection onto the same subspace, is a separate
signal-separation condition: it rules out the case in which the target signal itself is concentrated
in the factor directions that the method is designed to downweight.

Assumption~\ref{ass:weight_separation_main} is algorithmic. It requires that the spectral
transformation assign vanishing weights to the leading factor directions, while keeping the remaining
directions learnable. This follows from the usual spike--bulk behavior of \(X\). Under sub-Gaussian
assumptions on \(H\) and \(E\), non-asymptotic random matrix bounds imply that the noise bulk has
operator norm of order \(\sqrt n+\sqrt p\), whereas the strong factor component produces \(q\)
spiked singular values of order \(\sqrt{np}\) \citep{vershynin2020high}. Thus the leading singular values
separate from the bulk. For LAVA-type weights,
\begin{equation}
\label{eq:lava_weights}
w_i = \frac{n\lambda_2}{n\lambda_2+d_i^2},
\end{equation}
choosing \(n\lambda_2\) at the bulk scale, \(n\lambda_2\asymp(\sqrt n+\sqrt p)^2\), gives
\(w_i\to0\) on the spiked directions and weights bounded away from zero on the bulk directions.
The same spike--bulk separation also explains the Trim transform: the leading singular values are
capped at a bulk scale, so their ratios \(\tilde d_i/d_i\) vanish on the spikes, while remaining
bounded away from zero on non-spiked directions.

The theorem then isolates the deconfounding mechanism. Spectral weighting does not remove the
confounding component; instead, it slows the optimization dynamics in directions where the fitted
confounding term is concentrated. This creates a window in which the signal directions have
been learned, while the downweighted confounding directions have not. Early stopping is therefore essential, and the theorem formalizes this as an intermediate
stopping regime: the number of iterations \(m_p\) must diverge so that the non-confounding directions are learned,
\(m_p\nu w_M\to\infty\), but not so fast that the downweighted confounding directions are fitted,
\(m_p\nu w_H\to0\). We empirically validate the importance of this intermediate stopping regime in Section~\ref{sec:quant_sim}.

The proof is given in Appendix~\ref{app:pathwise_unbiasedness}. The verification of
Assumptions~\ref{ass:reduced_model_energy_main}--\ref{ass:weight_separation_main} under a dense
factor model and either the Trim or LAVA transform is given in
Appendix~\ref{app:factor_model_justification}. We remark that Theorem~\ref{thm:pathwise_unbiasedness_main} is a bias statement: it controls the conditional mean of the boosting iterate. Full prediction consistency additionally requires variance control; for the OLS base learner considered here, this requires additional effective-dimension conditions, with \(k/n\to0\), being a simple sufficient condition; see remark~\ref{rem:variance_pathwise} in Appendix~\ref{app:pathwise_unbiasedness}.

The analysis uses OLS base learners, which makes the boosting recursion exactly diagonal in the
left singular-vector basis of \(X\). For nonlinear learners such as trees, this exact diagonalization is
no longer available. Nevertheless, the theorem provides the guiding optimization picture: the spectral
loss changes the pseudo-response so that residual variation aligned with the confounding subspace is
learned more slowly. We use this insight to motivate the nonlinear spectral boosting procedure in the
next section.

\section{Spectral Deconfounding for Nonlinear Models via Tree-Boosting}
\label{sec:non-linear}

We now move beyond linear models and allow for a general nonlinear relationship between the covariates \(X\) and the response \(Y\). Our proposed solution for spectral deconfounding extends the results from Section~\ref{sec:mm} to nonlinear mixed-effects models. Moreover, we propose to choose the spectral shrinkage via empirical Bayes and to tune the amount of regularization using a certain form of cross-validation. We consider the structural equation model for $n$ i.i.d. observations stacked in matrix form as
\begin{align}
X &= H \Gamma + E, \label{eq:sem_x_non-linear} \\
Y &= f_0(X) + H \delta + \epsilon.
\label{eq:sem_y_non-linear}
\end{align}
The dimensions of all matrices and vectors, as well as the independence assumptions on \(E\), \(H\), and \(\epsilon\), are the same as in Section~\ref{sub:linear}.
Here \(f_0:\mathbb{R}^p\to\mathbb{R}\) denotes the target nonlinear regression function, and \(f_0(X)\) denotes its row-wise evaluation on the design matrix \(X\), that is,
$
    f_0(X) := \bigl(f_0(x_1),\ldots,f_0(x_n)\bigr)^\top \in \mathbb{R}^n .
$
As in the linear setting, confounding enters additively through the latent term \(H\delta\), while the goal is to recover \(f_0\) without explicitly estimating \(H\).
In the nonlinear setting, the observable regression function, that is the function recovered by population least-squares regression of Y on X, is
\[
\mathbb E[Y\mid X]=f_0(X)+m(X), \qquad m(X):=\delta^\top \mathbb E[H\mid X],
\]
and the, potentially nonlinear, part of \(m(X)\) that is absorbed in estimation depends on the chosen model class for $f$. In particular, under a linear model class this reduces to the linear projection \(Xb_0\).

Similarly to Section \ref{sec:mm}, we consider the model
\[
Y = f(X) + X b + \epsilon,
\]
where \(f:\mathbb{R}^p\to\mathbb{R}\) is a possibly nonlinear function and \(f(X)\in\mathbb{R}^n\) denotes its row-wise evaluation on the observed covariates,
\(b \sim \mathcal N(0,\sigma_r^2 I_p)\) is a random effect, and \(\epsilon\sim\mathcal N(0,\sigma_e^2 I_n)\). Again, integrating out \(b\) yields the Gaussian marginal model
$
Y \mid X, f \sim \mathcal N\!\bigl(f(X), \Sigma\bigr)
$, $
\Sigma = \sigma_r^2 X X^\top + \sigma_e^2 I_n.
$
Hence the marginal log-likelihood is
\begin{equation}
\label{eq:spec_loss}
\ell(f, \Sigma)
=
-\frac12 (Y-f(X))^\top \Sigma^{-1} (Y-f(X))
-\frac12 \log |\Sigma|
+\mathrm{const}.
\end{equation}
The log-determinant does not depend on \(f(\cdot)\), so maximizing the marginal likelihood with respect to $f(\cdot)$ is equivalent to minimizing the spectral loss in \eqref{spec_loss} with $X\beta$ replaced by $f(X)$.

\subsection{Gradient tree-boosting with the spectral loss}
\label{sec:boosting}
Next, we propose to estimate the potentially nonlinear signal $f_0(\cdot)$ via gradient tree-boosting by minimizing the spectral loss induced by the quadratic term in Equation~\ref{eq:spec_loss}. Gradient boosting constructs a function \(f(\cdot)\) as an additive expansion of base learners,
\[
f(x)=\sum_{m=1}^M \nu\, h^{(m)}(x),
\]
where \(h^{(m)} \in \mathcal H\) are some base learners, in our case regression trees, and \(\nu>0\) is the learning rate. As mentioned in Section \ref{sec:opt_view}, the method can be viewed as stagewise gradient descent in a function space: at iteration \(m\), one fits a new base learner to approximate the negative gradient of the loss evaluated at the current fit \(f^{(m)}\).
We use a spectral loss. Writing \(r^{(m)}:=Y-f^{(m)}(X)\) for the residual vector at iteration \(m\), the gradient with respect to the prediction vector is $
\nabla_{f} L_F(f^{(m)}(\cdot))
=
- F^\top F\, r^{(m)}.
$
Hence boosting with the spectral loss differs from ordinary $L_2$-boosting only through the gradient: instead of fitting a tree to the raw residuals \(r^{(m)}\), one fits it to the filtered residuals \(F^\top F\,r^{(m)}\).

In the nonlinear setting, we define the spectral transformation as a full-space operator. Specifically,
\begin{equation}\label{def_F_non_lin}
F
=
U\operatorname{diag}\left(\sqrt{w_1},\ldots, \sqrt{w_k} \right)U^\top
+
(I_n-UU^\top).
\end{equation}
Thus, \(F\) rescales the left singular directions of \(X\) and acts as the identity on the orthogonal complement of \(\operatorname{col}(X)\). Consequently,
\[
F^\top F r^{(m)}
=
\sum_{i=1}^k  w_i
\langle u_i,r^{(m)}\rangle u_i
+
(I_n-UU^\top)r^{(m)} .
\]
This full-space definition does not affect the linear OLS base-learner analysis in Section~\ref{sec:opt_view}, since the OLS base learner depends only on the component of the pseudo-response lying in \(\operatorname{col}(X)\). For tree base learners, however, the fitted update is not restricted to \(\operatorname{col}(X)\), so the action of the spectral transformation on \(\operatorname{col}(X)^\perp\) must be specified. When the weights in \eqref{def_F_non_lin} are chosen as the LAVA weights given in \eqref{eq:lava_weights}, the definition in \eqref{def_F_non_lin} coincides with the LAVA transformation in \eqref{eq:lava_transform} and the mixed-model whitening transformation in Proposition \ref{prop:lava_mixed_model}, up to a global scaling factor. This factor does not affect the corresponding squared-loss minimizer and can be absorbed into the learning rate in boosting.

Algorithm~\ref{alg:spectral_boosting_fixed} summarizes gradient boosting with a fixed spectral transformation. At each iteration, the algorithm computes the filtered residuals \(F^\top F\,r^{(m)}\), fits a base learner to this pseudo-response, and updates the current prediction function by a learning-rate-scaled step.

\begin{algorithm}[ht!]
\caption{Gradient Boosting with Fixed Spectral Loss}
\label{alg:spectral_boosting_fixed}
\begin{algorithmic}[1]
\Require Data \((X,Y)\), spectral operator \(F\), number of iterations \(M\), learning rate \(\nu\)
\Ensure \(\hat f(\cdot)\)
\State Initialize \(f^{(0)}(x)\equiv 0\)
\For{\(m=0,\dots,M-1\)}
  \State Compute residuals: \(r^{(m)}=Y-f^{(m)}(X)\)
  \State Compute negative spectral gradient: \(\mathbf g^{(m)}=F^\top F\,r^{(m)}\)
  \State Fit base learner \(h^{(m)}\in\mathcal H\) to \((X,\mathbf g^{(m)})\) by least squares
  \State Update \(f^{(m+1)}(x)=f^{(m)}(x)+\nu\,h^{(m)}(x)\)
\EndFor
\State \Return \(\hat f(\cdot)=f^{(M)}(\cdot)\)
\end{algorithmic}
\end{algorithm}

\subsection{Empirical-Bayes spectral shrinkage and regularization}
Spectral deconfounding involves two conceptually distinct types of tuning choices. The first concerns the \emph{spectral transformation} itself, namely how the geometry of the design is modified. The second concerns the \emph{regularization} of the estimator fitted under that geometry, such as the lasso or ridge penalty in linear models or the number of boosting iterations in our nonlinear procedure. Separating these two roles is important both conceptually and algorithmically.

\subsubsection{Choosing the spectral transformation via Empirical Bayes}
Different spectral deconfounding methods involve different transformation parameters. The Trim transform \citep{cevid2020spectral} is essentially tuning-free, since it modifies the spectrum of \(X\) according to a fixed deterministic rule. Principal component adjustment depends on the number of principal components removed or controlled for, which is often motivated by factor-model arguments. In practice, however, selecting this latent dimension is difficult, since the asymptotic regimes under which the number of factors can be identified do not directly translate to finite-sample problems \citep{bai2003inferential}. By contrast, LAVA-type transformations are parameterized through a variance-ratio or shrinkage level that controls how strongly high-variance spectral directions are attenuated.

In our approach, this transformation is selected through a mixed-model formulation. We model the dense confounding component as a random effect \(X b\), with \(b \sim \mathcal N(0,\sigma_r^2 I_p)\), and estimate the variance components \(\theta=(\sigma_r^2,\sigma_e^2)\) by empirical Bayes \citep{robbins1964empirical}, that is, by maximizing the marginal log-likelihood obtained after integrating out the random effect:
\begin{equation}
\ell(\theta;f)
=
-\frac12 \log |\Sigma_\theta|
-\frac12 (Y-f(X))^\top \Sigma_\theta^{-1}(Y-f(X))
+\mathrm{const},
\label{eq:eb_objective}
\end{equation}
where
\[
\Sigma_\theta = \sigma_r^2 X X^\top + \sigma_e^2 I_n.
\]
For fixed \(f(\cdot)\), this yields the spectral operator \(F=\Sigma_\theta^{-1/2}\). Thus, rather than specifying a trimming threshold or the number of removed principal components a priori, we estimate the spectral geometry from the data through the variance components.

In practice, we optimize \(\theta\) and \(f(\cdot)\) alternately. For fixed \(f(\cdot)\), we update the variance components by maximizing \(\ell(\theta;f(\cdot))\). For fixed \(\theta\), we update \(f(\cdot)\) by gradient boosting under the spectral loss induced by \(F^\top F=\Sigma_\theta^{-1}\). When \(p<n\), the application of \(\Sigma_\theta^{-1}\) can be computed efficiently via the Sherman--Morrison--Woodbury identity, while when \(p \geq n\), one may work directly with the \(n \times n\) covariance matrix \(\Sigma_\theta\). For large \(p\) and \(n\), scalability can be obtained using sparse Gaussian-process approximations such as Vecchia approximations \citep{Katzfuss_2020}, which are implemented in our software (for both the linear kernel and the extension below to nonlinear kernels).

\begin{algorithm}[ht!]
\caption{Spectral Boosting with Empirical Bayes}
\label{alg:spectral_boosting_eb}
\begin{algorithmic}[1]
\Require Data \((X,Y)\), initial \(\theta^{(0)}=(\sigma_r^{2(0)},\sigma_e^{2(0)})\), number of iterations \(T\), learning rate \(\nu\)
\Ensure \(\hat f\), \(\hat\theta\)
\State Initialize \(f^{(0)}(x)\equiv 0\)
\For{\(t=0,\dots,T-1\)}
  \State Form \(\Sigma(\theta^{(t)})=\sigma_r^{2(t)}XX^\top+\sigma_e^{2(t)}I_n\)
  \State \textbf{Boosting step:} run Algorithm~\ref{alg:spectral_boosting_fixed} for \(1\) iteration
  \Statex \hspace{1.5em} using \(F^\top F=\Sigma(\theta^{(t)})^{-1}\), initialized at \(f^{(t)}\), to obtain \(f^{(t+1)}\)
  \State \textbf{EB step:} update \(\theta^{(t+1)} \in \arg\max_{\theta}\,\ell(\theta;f^{(t+1)})\)
\EndFor
\State \Return \(\hat f=f^{(T)}\), \(\hat\theta=\theta^{(T)}\)
\end{algorithmic}
\end{algorithm}

\subsubsection{Choosing the stopping iteration via BLUP-corrected cross-validation}
Regularization plays a different role. Whereas the spectral transformation determines which directions are downweighted, regularization determines how far the estimator is allowed to move along the resulting optimization path. In the linear spectral deconfounding literature, \citet{cevid2020spectral} tune the estimator regularization, for example the lasso penalty, by cross-validation and then use a conservative selection rule. This concerns the complexity of the estimator, not the transformation itself. By contrast, mixed-model formulations are typically used to estimate variance components, but do not by themselves emphasize the additional regularization needed when the goal is deconfounding rather than pure prediction. A similar point applies to spectrally deconfounded random forests \citep{Ulmer03042026}, where the transformation is fixed in advance and the role of regularization is not explicit.

In our boosting-based procedure, regularization is governed primarily by early stopping, that is, by the number of boosting iterations.
For component-wise linear $L_2$-boosting, the iteration number $m$ plays the role of an inverse regularization parameter: increasing $m$ weakens the regularization. In the infinitesimal-step-size limit, the resulting forward-stagewise path is closely related to, and under additional conditions coincides with, the lasso/LARS path \citep{Efron_2004, Hastie_2007}.
We select this number of iterations by \(K\)-fold cross-validation with the negative log-likelihood as criterion. Importantly, prediction on the validation data is done using both the fixed-effects term and the best linear unbiased predictor (BLUP) for the random-effects component. For a training fold \((X_{\mathrm{tr}}, Y_{\mathrm{tr}})\) and validation covariates \(X_{\mathrm{val}}\), we predict
\[
\hat{Y}_{\mathrm{val}}
=
\hat f(X_{\mathrm{val}})
+
\sigma_r^2 X_{\mathrm{val}} X_{\mathrm{tr}}^\top
\Sigma_{\theta,\mathrm{tr}}^{-1}
\bigl(Y_{\mathrm{tr}}-\hat f(X_{\mathrm{tr}})\bigr),
\qquad
\Sigma_{\theta,\mathrm{tr}}
=
\sigma_r^2 X_{\mathrm{tr}} X_{\mathrm{tr}}^\top + \sigma_e^2 I_n.
\]

The inclusion of both the fixed- and random-effects terms for prediction matters for model selection. If one predicts only with \(\hat f(X_{\mathrm{val}})\), then cross-validation can favor estimators that absorb confounding structure into $\hat f(\cdot)$, since the same confounding pattern is typically present in both training and validation folds. In that setting, conservative selection rules such as the one-standard-error rule can be helpful \citep{cevid2020spectral}. In our formulation, by contrast, the posterior random-effect term is available at prediction time and absorbs \(X\)-aligned residual structure. This makes standard cross-validation based on prediction error substantially better aligned with the goal of selecting the stopping time for $\hat f(\cdot)$. In Section~\ref{sec:quant_sim}, we show empirically that this procedure performs comparably to an oracle stopping rule based on the error with respect to the true target function \(f_0(\cdot)\).

\section{Extensions: Non-Gaussian Likelihoods and Nonlinear Confounding}
\label{sec:generalized_lava_classification}

We next extend our proposed spectral boosting method beyond the Gaussian likelihood. Our starting point is the generalized LAVA estimator of \citet{wang2025latent}, which replaces the unknown dense confounding contribution by a ridge-regularized bias proxy in the column space of the data matrix \(X\). We then reinterpret this proxy through an empirical-Bayes (EB) mixed-model view, in which the dense component is modeled as a Gaussian random effect.

\subsection{General likelihoods, generalized LAVA, and a mixed model interpretation}
\label{sec:classif_setup}

Let \(Y=(y_1,\ldots,y_n)^\top\) denote the observed outcomes and \(X \in \mathbb R^{n\times p}\) the covariates. Conditionally on a linear predictor \(\eta \in \mathbb R^n\), we assume independent observations with likelihood
\[
p(Y\mid \eta)=\prod_{i=1}^n p(y_i\mid \eta_i).
\]
Equivalently, one may specify a mean parameter \(\mu_i\) and a link function \(g\) through \(g(\mu_i)=\eta_i\). This includes, for example, Gaussian regression with identity link, Bernoulli likelihoods with logistic or probit links, and Poisson likelihoods with log link. We consider the confounded generalized model
\begin{equation}
\label{eq:logit_confounded_structural}
y_i \mid \eta_i \sim p(\,\cdot \mid \eta_i),
\qquad
\eta = f_0(X) + H\delta,
\qquad
X = H\Gamma + E.
\end{equation}

Generalized LAVA \citep{wang2025latent} replaces the unknown confounding contribution \(H \delta\) by a dense proxy in the column space of \(X\),
\begin{equation}
\label{eq:logit_proxy_bias}
\eta = X \beta + X b,
\end{equation}
where \(b \in \mathbb R^p\) is a dense bias parameter. Recall that, under strong and dense confounding, spurious structure is expected to concentrate along high-variance directions of \(X\), so the proxy \(X b\) is designed to absorb \(X\)-aligned variation. Writing
\[
\rho_i(\eta_i) := -\log p(y_i\mid \eta_i)
\]
for the negative log-likelihood contribution, generalized LAVA estimates \((\beta,b)\) via
\begin{equation}
\label{eq:genlava_objective_matrix}
(\hat{\beta},\hat{b}) \in \arg\min_{\beta,b}
\left\{
\frac{1}{n}\sum_{i=1}^n \rho_i\bigl((X \beta+X b)_i\bigr)
+ \lambda_1 \|\beta\|_1
+ \lambda_2 \|b\|_2^2
\right\}.
\end{equation}

We propose to estimate the signal component and the regularization strength by empirical Bayes, treating the dense proxy as a Gaussian random effect 
\(b \sim \mathcal N(0,\sigma_r^2 I_p)\).
Allowing for nonlinear functions, we write the predictor as
\(\eta = f + X b\),
where \(f \in \mathbb R^n\) may be linear, \(f=X \beta\), or nonlinear, \(f_i=f(x_i)\).
Empirical Bayes targets the marginal likelihood
\[
p(Y \mid f,\sigma_r^2)
=
\int p(Y \mid \eta=f+X b)\, \pi(b \mid \sigma_r^2)\, db,
\]
which is not available in closed form for general likelihoods.
We therefore use the Laplace approximation \citep{tierney1989fully}
\begin{equation} \label{eq:laplace_marginal_like_classif_short}
\log p(Y \mid f,\sigma_r^2) \approx \log p(Y \mid \eta=f+X \hat{b}) - \frac{1}{2\sigma_r^2}\|\hat{b}\|_2^2 
-\frac{1}{2}\log\det\!\Bigl(\sigma_r^2X^\top W(\hat{b},f)X+I_p\Bigr) +\mathrm{const},
\end{equation}
where \(\hat{b}\) is the mode
\begin{equation}
\hat{b}(f,\sigma_r^2) = \arg\max_{b} \log p(Y \mid \eta=f+X b) - \frac{1}{2\sigma_r^2}\|b\|_2^2,
\end{equation}
and
\[
W(\hat b,f)
=
-\nabla_\eta^2 \log p(Y\mid \eta)\big|_{\eta=f+X\hat b}.
\]
More generally, Laplace-type approximations can also be constructed for non-smooth likelihoods by replacing the Hessian with an appropriate local curvature \citep{nava2026laplaceapproximationsmixedeffectsgaussian}.

Note that, for fixed \((f,\sigma_r^2)\), the mode \(\hat{b}(f,\sigma_r^2)\) is the dense generalized LAVA proxy. Unlike the Gaussian case, it is not generally available in closed form, since the log-likelihood need not be quadratic. It is instead computed by Newton's method, equivalently by ridge-regularized iteratively reweighted least squares (IRLS): at a current iterate \(b^{(t)}\), let \(\eta^{(t)}=f+Xb^{(t)}\), define the score \(s^{(t)}=\nabla_\eta \log p(Y\mid \eta)\vert_{\eta=\eta^{(t)}}\), and set \(W^{(t)}=-\nabla_\eta^2 \log p(Y\mid \eta)\vert_{\eta=\eta^{(t)}}\). The second-order Taylor expansion of the log-likelihood at \(\eta^{(t)}\) yields the IRLS pseudo-response
\[
z^{(t)}
=
\eta^{(t)}+\bigl(W^{(t)}\bigr)^{-1}s^{(t)}.
\]
The Newton step is therefore obtained by solving the weighted ridge system
\begin{equation}
\label{eq:b_update_linear_system_clean}
\left(X^\top W^{(t)} X + \sigma_r^{-2}I_p\right)b^{(t+1)}
=
X^\top W^{(t)}\left( z^{(t)}-f\right).
\end{equation}
This shows that each Newton step is a ridge-penalized weighted least-squares projection of the current working residual \( z^{(t)}-f\) onto \(\operatorname{col}(X)\). This update allocates \(X\)-aligned variation to the dense proxy. The full mode \(\hat{b}(f,\sigma_r^2)\) is obtained by iterating these weighted ridge updates until convergence. This makes the connection to generalized LAVA precise at the level of the posterior mode: for fixed \((f,\sigma_r^2)\), the dense component is still ridge-regularized, but now after a local quadratic approximation of the likelihood.

Empirical Bayes selects \((f,\sigma_r^2)\) by optimizing the Laplace-approximated log-marginal likelihood \eqref{eq:laplace_marginal_like_classif_short}. In our implementation, we update the nonlinear signal \(f\) by gradient boosting with respect to the corresponding negative Laplace-approximated log-marginal loss; explicit expressions of the derivatives are given in \citet{9759834}. This yields a stagewise procedure in which boosting updates refine \(f\) given the current proxy, and \(\sigma_r^2\) is updated by one-dimensional optimization of \eqref{eq:laplace_marginal_like_classif_short}. The variance component therefore controls again how much \(X\)-aligned structure is allocated to the dense proxy versus the nonlinear signal. We demonstrate the effectiveness of the proposed approach for binary classification in Section \ref{sec:classif_sim}.

\subsection{Nonlinear kernel-based deconfounding}
\label{sec:non-linear_extension}

We now extend the spectral deconfounding framework beyond the linear confounding regime. Consider the structural model
\[ X = \phi(H) + E, \qquad Y = f_0(X) + g(H) + \varepsilon, \]
where $H \in \mathbb R^{n \times q}$ is a latent confounder, $\phi$ is a (possibly nonlinear) factor map, and $g(H)$ is an arbitrary latent confounding contribution. The observable regression function can again be written as
\[
\mathbb E[Y \mid X] = f_0(X) + m(X), \qquad m(X) := \mathbb E[g(H)\mid X].
\]
In the linear-Gaussian setting, where $\phi(H)=H \Gamma$ and $g(H)=H \delta$, the conditional expectation $\mathbb E[H\mid X]$ is linear in $X$. Consequently,
$
m(X)=X b
$
for some coefficient vector $b$, so the confounding bias lies exactly in the linear span of $X$.

If both $\phi(H)$ and $g(H)$ admit a common finite-dimensional basis expansion
\[
\phi(H)=\Gamma B(H),
\qquad
g(H)=\delta^\top B(H),
\]
for some feature map $B:\mathbb R^q \to \mathbb R^K$ with $K \ll p$, then the problem reduces to a linear confounding model in the expanded latent variables $B(H)$. In this case, linear spectral deconfounding remains sufficient \citep{Ulmer03042026}. In contrast, when no such shared finite representation exists, $m(X)$ is, in general, a nonlinear function of $X$. Importantly, the practically relevant bias is not $m(X)$ itself, but its projection onto the model class used to estimate $f_0$. Under the squared loss, the fitted predictor absorbs only the component of $m(X)$ that lies in the hypothesis space. In particular, for linear models, the relevant bias reduces to the linear projection $X b$, while for nonlinear learners, richer components of $m(X)$ may be representable and thus absorbed. This highlights a limitation of linear spectral transformations: they are designed to attenuate $X b$, but do not directly address nonlinear components of $m(X)$ that become relevant for more expressive learners.

To account for nonlinear confounding, we extend the spectral transformation to a more flexible reproducing kernel Hilbert space (RKHS). Instead of modeling the dense component via a random effect with covariance $X X^\top$, we consider a kernel-based random effect with covariance matrix $K \in \mathbb R^{n \times n}$ with entries
\[
K_{ij} = k_{\theta}(x_i, x_j),
\]
where $k_{\theta}(\cdot,\cdot)$ is a positive definite kernel with hyperparameters $\theta$.
Analogously to the linear case, this induces a spectral transformation based on the eigendecomposition of $K$. The resulting spectral loss is
\[
 L_F(f) 
= \frac{1}{2}\| F (Y - f(X)) \|_2^2,
\qquad
F = (\sigma_r^2 K + \sigma_e^2 I_n)^{-1/2}.
\]

The mechanism of this approach can be understood through a simple spectral argument. Let $(\lambda_j, v_j)$ denote the eigenpairs of $K$, and write the confounding signal on the sample as
\[
m := (m(x_1),\dots, m(x_n))^\top = \sum_{j=1}^n \alpha_j v_j.
\]
Since $F$ acts diagonally in this basis, $F v_j = (\sigma_r^2 \lambda_j + \sigma_e^2)^{-1/2} v_j$, we obtain
\begin{equation}
\|F m\|_2^2
=
\sum_{j=1}^n \frac{\alpha_j^2}{\sigma_r^2 \lambda_j + \sigma_e^2}.
\label{eq:kernel_shrink_m}
\end{equation}

\noindent
Thus, spectral deconfounding is effective whenever the confounding signal $m$ is concentrated on eigenvectors associated with large eigenvalues of $K$, since these directions are downweighted most strongly. Kernel eigenvectors capture dominant \emph{nonlinear modes of variation} in the data. Equation \eqref{eq:kernel_shrink_m} shows that the method attenuates precisely those components of $m(X)$ that align with these dominant modes. Therefore, the kernel extension succeeds when the predictable confounding term is \emph{spectrally simple} with respect to the geometry of $X$. Finally, the kernel hyperparameters \(\theta\) can be estimated by empirical Bayes through maximization
of the marginal likelihood, exactly as in the linear random-effects model. 

\subsection{Software and reproducibility}
\label{sec:software}

The methods presented in this article are implemented in the \texttt{GPBoost} library, written in C++ with high-level Python and R interfaces; see \url{https://github.com/fabsig/GPBoost}. Code for reproducing the experiments in this paper is available at \url{https://github.com/AndreaThomNava/spectral-boosting}. All numerical experiments were run on a machine with an Intel Xeon E3-1284L v4 processor at 2.90\,GHz and 31\,GB of RAM.

\section{Simulation Studies}
\label{sec:experiments}
To investigate the proposed spectrally deconfounded gradient boosting algorithm, we follow the data-generating process of \citet{Ulmer03042026}. We begin by simulating a latent confounding structure that induces correlation in the covariates and enters the outcome through a shared linear effect. Specifically, we generate
$
X = E + H \Gamma,
E \in \mathbb R^{n\times p},\;
H \in \mathbb R^{n\times q},\;
\Gamma \in \mathbb R^{q\times p},
$
and define the latent predictor
$
\eta := f_0(X) + H \delta \in \mathbb R^n,
\delta \in \mathbb R^{q}.
$
The true causal function \(f_0\) is constructed using a Fourier basis:
\begin{equation}
f_0(X) := \sum_{j=1}^{p} \mathbbm{1}_{\{j \in J_s\}} \sum_{k=1}^{K} \left( a_{j,k} \cos(0.2k \cdot x_j) + b_{j,k} \sin(0.2k \cdot x_j) \right),
\label{eq:conf_nn_model}
\end{equation}
where \(J_s\) is a randomly selected subset of \(\{1, \ldots, p\}\) of size four, representing the true causal features. The entries of \(E\), \(H\), \(\delta\), and \(\Gamma\) are sampled independently from \(\mathcal N(0, 1)\). For the four causal parents, we sample the coefficients \(a_{j,k}\) and \(b_{j,k}\) uniformly on \([-1, 1]\), with the number of basis functions fixed at \(K = 2\). We then obtain the observed response according to the learning task. In the regression setting, we add independent Gaussian noise,
$ Y = {\eta} + {\varepsilon},
{\varepsilon} \sim \mathcal N(0, 0.1^2\, I_n).$
In the classification setting, we map the latent predictor through the logistic link and sample Bernoulli labels accordingly.
\subsection{Qualitative results}

The goal of this section is to illustrate how spectral boosting changes the learned signal under latent confounding in terms of feature importance and recovery of the shape of the true causal function. We generate data from the confounded nonlinear model \eqref{eq:conf_nn_model} with \(n=1{,}000\) training observations, \(p=100\) covariates, and \(q=20\) latent confounders for the regression setting, and \(n = 10{,}000\), \(p = 50\), and \(q = 20\) for the classification setting. We fit (i) our proposed spectral boosting method and (ii) traditional gradient tree boosting, both with a learning rate of \(\nu = 0.05\) and a tree depth of three. The variance components entering the covariance matrix \(\Sigma_{\theta}\) are selected by empirical Bayes, while the number of boosting iterations is chosen by 2-fold cross-validation, using as selection criterion the mean squared error and the negative log-likelihood for the regression and classification settings, respectively.

\paragraph{Regression setting.}
The results in Figure~\ref{fig:imp_pdp_regression} suggest two qualitative advantages of spectral boosting over standard boosting under confounding. First, the split-based feature-importance analysis in the left plot shows that spectral boosting assigns substantially more importance to the four causal features, whereas standard boosting spreads importance more diffusely across many features, consistent with a stronger reliance on confounding-driven associations. 
Second, the partial dependence plots in the right panel show that spectral boosting better recovers the target shape on the causal variables and is less prone to estimating nonzero effects on non-causal variables. The displayed variables are selected by average split-gain importance rank across spectral boosting and standard boosting, separately among causal and non-causal covariates.
\begin{figure}[ht!]
    \centering
    \includegraphics[width=0.95\linewidth]{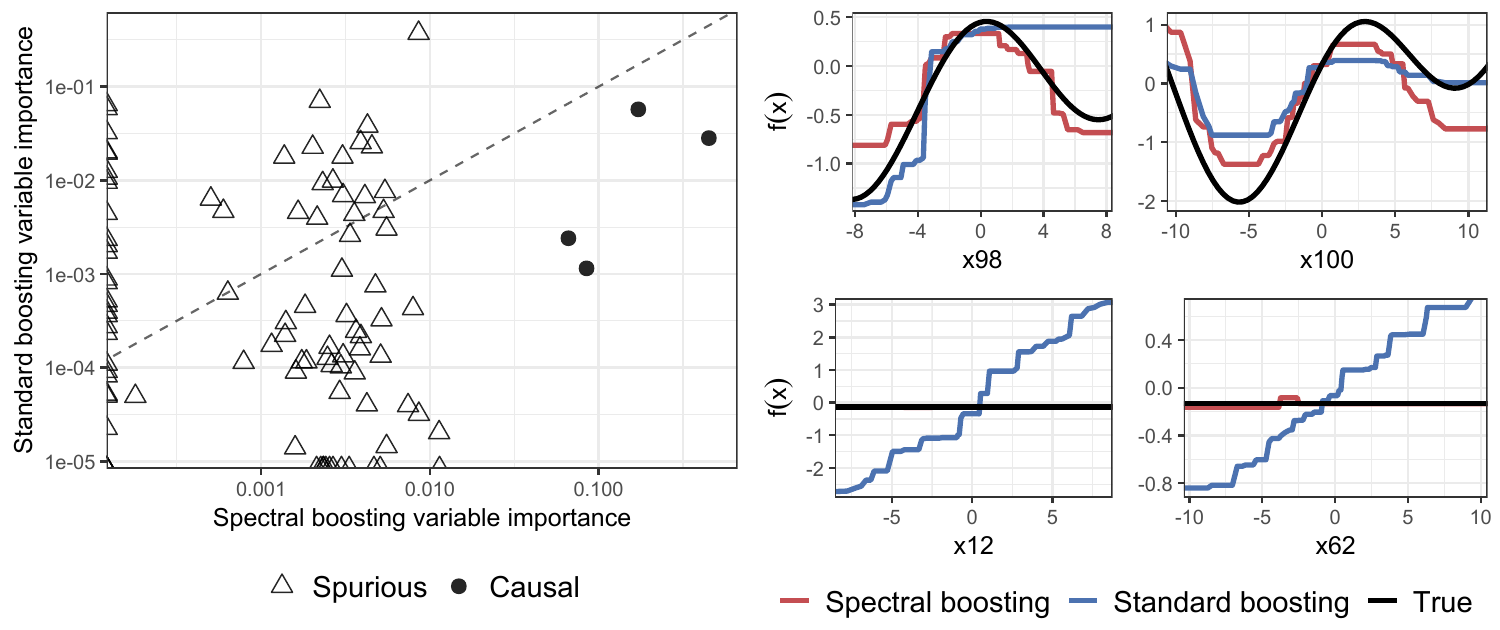}
    \caption{Regression setting.
    Left: Log--log scatter of split-based variable importance; 
    causal features are highlighted. Right: partial dependence plots for two causal and two non-causal covariates with the highest average variable importance rank across the two methods. The top row shows causal covariates (x98 and x100); the bottom row shows non-causal covariates (x12 and x62), for which the ground-truth effect is zero.
    }
    \label{fig:imp_pdp_regression}
\end{figure}

\paragraph{Classification setting.}
Figure~\ref{fig:imp_pdp_classification} shows a qualitatively similar pattern in the classification setting, although the differences are somewhat less pronounced than in regression. This is expected, since recovering the underlying signal is generally more difficult in classification even in the absence of confounding. As in the regression case, the feature-importance analysis shows that spectral boosting places more emphasis on the four truly causal covariates. The partial dependence plots in the right panel indicate that the two methods recover broadly similar effects on the causal variables, while spectral boosting is less prone to estimating nonzero effects on non-causal variables. The displayed variables are selected using the same criterion as in Figure~\ref{fig:imp_pdp_regression}.

\begin{figure}[ht!]
    \centering
    \includegraphics[width=0.95\linewidth]{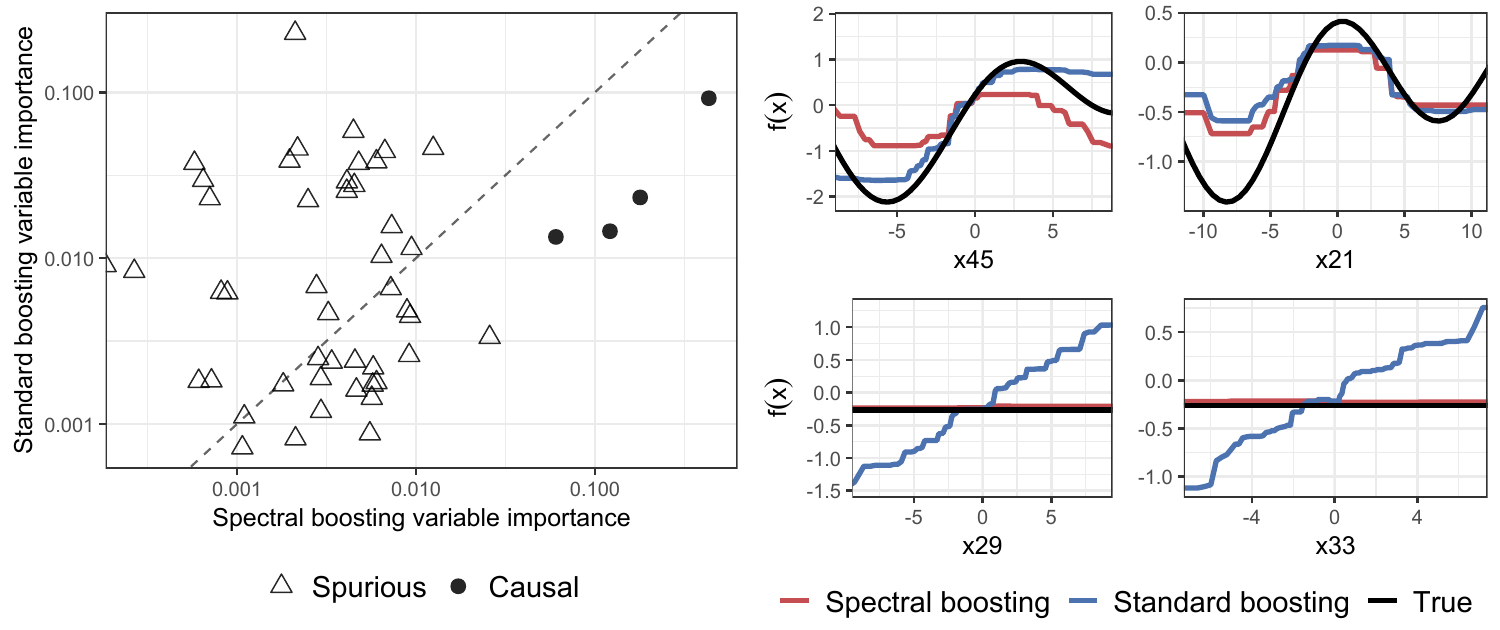}
    \caption{Classification setting. Left: Log--log scatter of split-based variable importance; 
    causal features are highlighted. Right: partial dependence plots for two causal and two non-causal covariates with the highest average variable importance rank across the two methods. The top row shows causal covariates (x45 and x21); the bottom row shows non-causal covariates (x29 and x33), for which the ground-truth effect is zero.
    }
    \label{fig:imp_pdp_classification}
\end{figure}

\subsection{Quantitative results}\label{sec:quant_sim}

To quantify the accuracy in estimating the causal function, we follow the simulation setting described above. In each repetition, we draw a causal function \(f_0\) from the Fourier basis expansion and generate data from the confounded model. Unless stated otherwise, the default dimensions are \(n=1{,}000\), \(p=250\), and \(q=20\) for the regression setting, and \(n=5{,}000\), \(p=50\), and \(q=20\) for the classification setting. In all cases, only four covariates are active, i.e., only four variables influence the response. We then vary one dimension at a time---the sample size \(n\), the number of covariates \(p\), and the latent confounder dimension \(q\)---while keeping the remaining quantities fixed at their default values. Each configuration is repeated 50 times, with the full data-generating process redrawn independently in every repetition. Performance is evaluated on a test set of size \(n_{\text{test}}=500\) using the mean squared error with respect to the true causal function:
\[
\mathrm{MSE}_f
:=
\frac{1}{n_{\text{test}}}
\sum_{i=1}^{n_{\text{test}}}
\bigl(f_0(x_{\text{test},i}) - \hat f(x_{\text{test},i})\bigr)^2.
\]

We benchmark the proposed spectrally deconfounded boosted trees (SpecBoost) against standard $L_2$-gradient-boosted trees (Boosting) and, in the regression setting, against random forests (Random Forest) using the \texttt{ranger} implementation \citep{wright2019package} and the spectrally deconfounded random forest of \citet{Ulmer03042026} (SDF). Both SpecBoost and Boosting are implemented using the \texttt{GPBoost} library \citep{sigrist2021gpboost}. This comparison isolates the effect of the spectral transformation within boosting and positions our method relative to the main existing nonlinear spectral deconfounding baseline. In classification, we report results for Spectral Boosting and Standard Boosting only, since SDF is designed for squared-loss regression and is not directly applicable to classification. In the regression setting, we additionally compare the runtime of the competing methods in order to assess the computational cost of spectral deconfounding.

\subsubsection{Regression setting}

Figure~\ref{fig:mse_grid} shows that the proposed SpecBoost method consistently achieves lower estimation error than standard boosting, and a similar pattern holds for SDF relative to standard random forests. The sample size has comparatively little effect on the relative ranking of the methods, whereas the number of features and the latent confounder dimension have a more pronounced impact. As the number of features increases, the benefits of spectral deconfounding become more evident, indicating that deconfounding is particularly useful in higher-dimensional settings. With respect to the latent confounder dimension, standard boosting performs best in the unconfounded case, as expected, but the loss relative to SpecBoost is small compared with the gains achieved by SpecBoost once confounding is present; see Appendix~\ref{app:no_confounding_regression} for a more thorough ablation in the absence of confounding. The results also suggest that the Trim transform used in SDF is less effective in regimes with relatively few features, in particular when the number of features is smaller than the confounder dimension (middle plot in Figure~\ref{fig:mse_grid}), or when the confounder dimension approaches the feature dimension (right plot in Figure~\ref{fig:mse_grid}).

\begin{figure}[ht!]
    \centering
    \includegraphics[width=1.\linewidth]{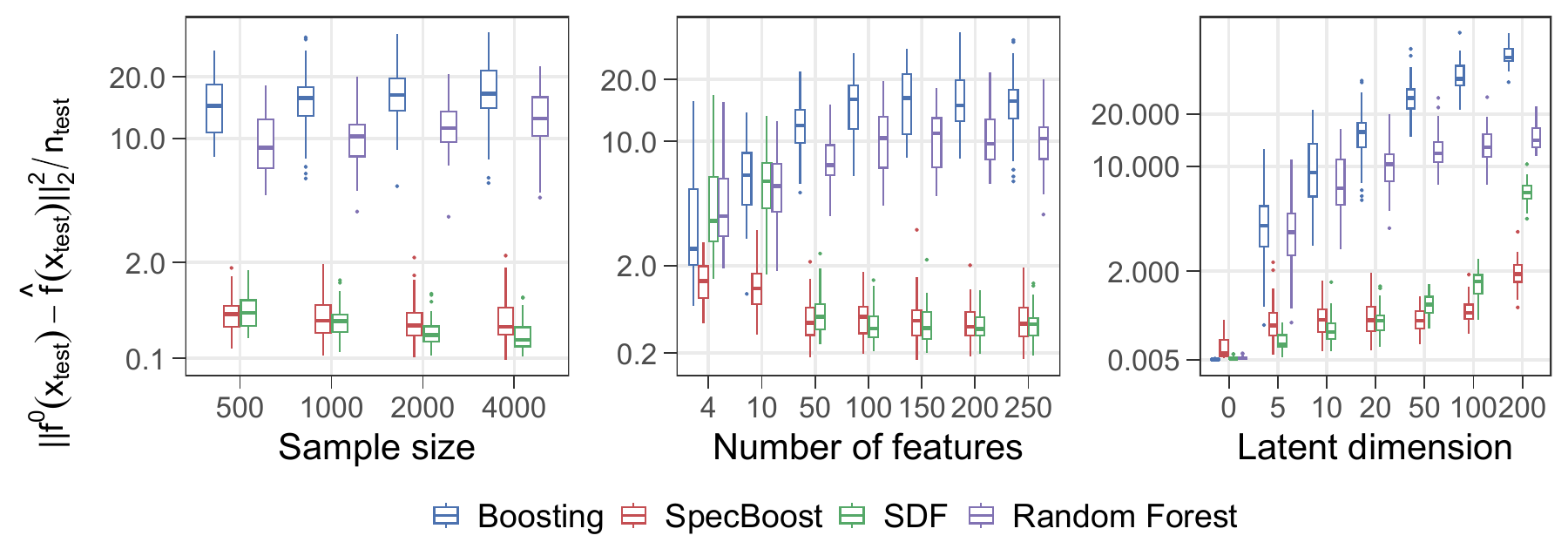}
    \caption{Regression setting. Boxplots of the mean squared error for the different methods across simulation scenarios, displayed on a \(\log(1+x)\) scale on the y-axis. 
    }
    \label{fig:mse_grid}
\end{figure}

Figure~\ref{fig:runtime} reports the runtimes across the same simulation settings. As expected, the spectral methods are computationally more demanding than their non-spectral counterparts, since they require additional matrix factorizations. At the same time, SpecBoost is several orders of magnitude faster than SDF across all scenarios. In particular, the runtime of SDF is strongly affected by the sample size, whereas our implementation scales substantially better in \(n\), making spectral deconfounding feasible for larger regression problems.

\begin{figure}[ht!]
    \centering
    \includegraphics[width=1\linewidth]{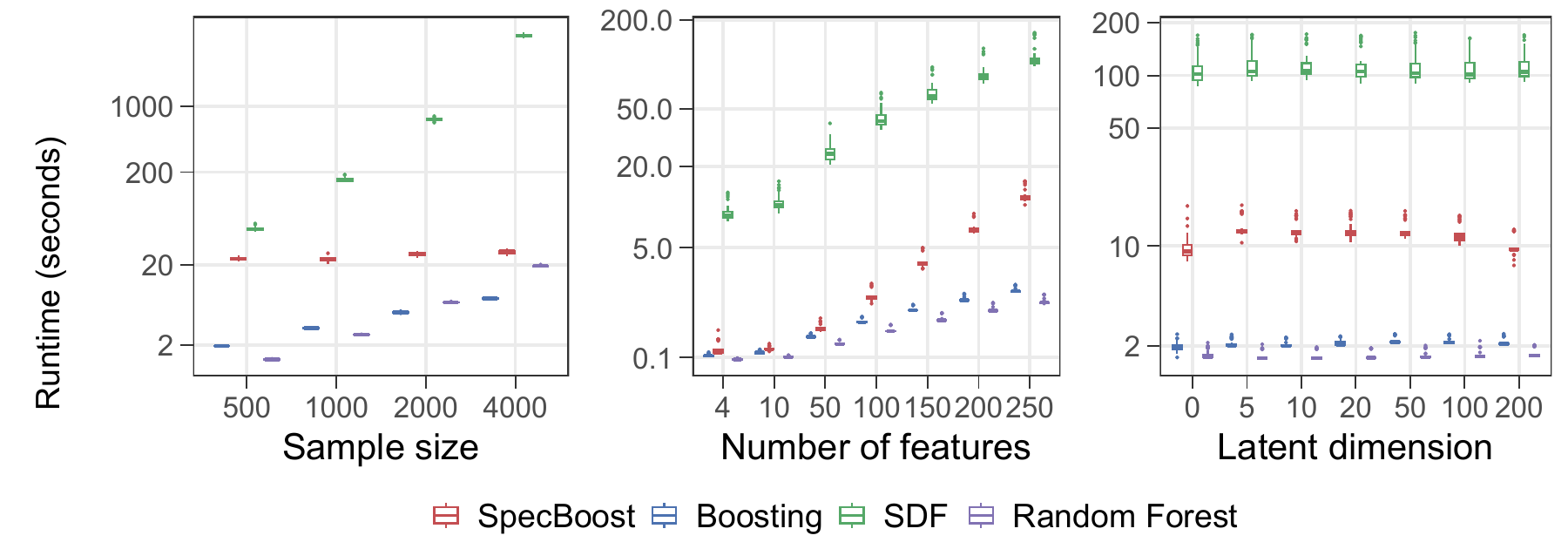}
    \caption{Regression setting.
    Boxplots of the runtime for the different methods across simulation scenarios,  displayed on a \(\log(1+x)\) scale on the y-axis. 
    }
    \label{fig:runtime}
\end{figure}

Figure~\ref{fig:iter_grid} compares the number of boosting iterations selected by cross-validation with an oracle stopping rule that minimizes the estimation error with respect to the true target function \(f_0\). The proposed cross-validation rule, which forms predictions using both the fixed-effects component and the posterior random effect, tracks the oracle stopping rule closely across all simulation scenarios. By contrast, standard boosting tends to overfit to the confounding structure and therefore selects too many iterations. 
The spectral method can be run for more iterations than standard boosting before overfitting to confounding becomes severe, but early stopping remains important: if the procedure is run too long, later iterations eventually begin to absorb confounding-driven residual variation.
As expected, the sample size has relatively little effect on the selected stopping time. In contrast, the number of features and the latent confounder dimension have a more visible impact.

\begin{figure}[ht!]
    \centering
    \includegraphics[width=1\linewidth]{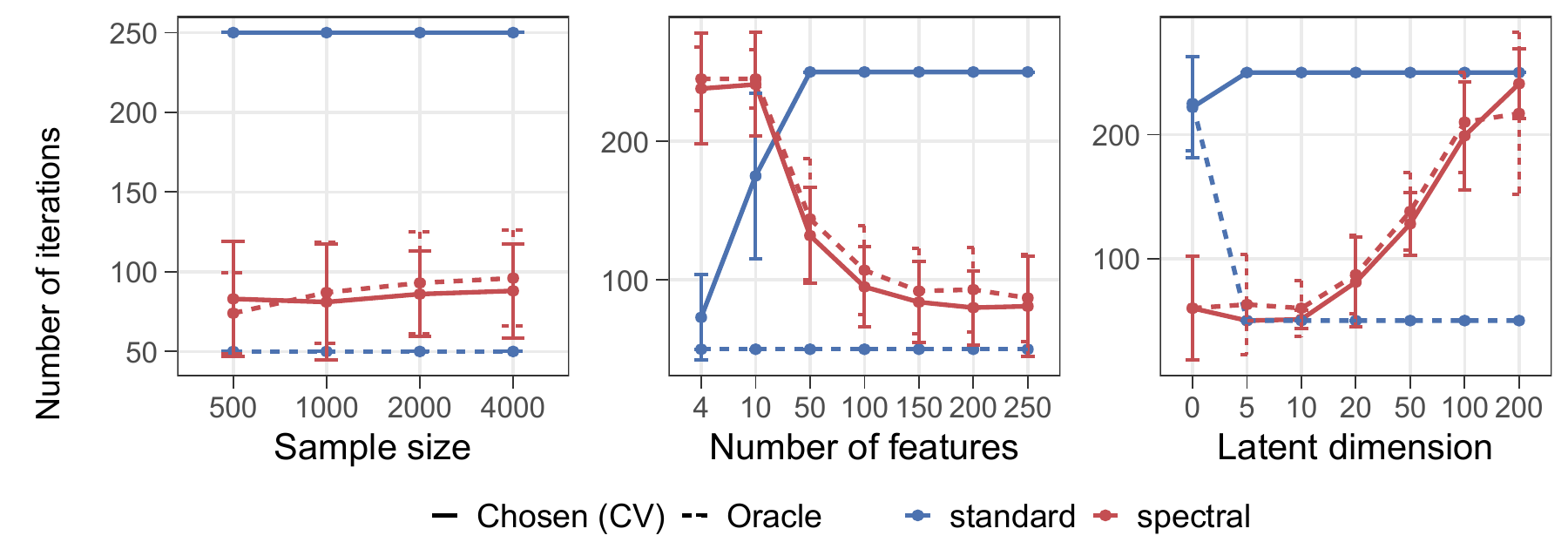}
    \caption{Regression setting.
    Number of boosting iterations selected by the proposed cross-validation rule
    (solid lines) and by the oracle rule with access to the true target function \(f_0\) (dashed lines).}
    
    \label{fig:iter_grid}
\end{figure}

Figure~\ref{fig:pars_grid} shows the estimated random-effect variance component \(\sigma_r^2\) selected by empirical Bayes. 
Since there is no ground-truth value for this parameter in the simulation design, we assess its behavior heuristically. In the absence of confounding, one would expect the random-effect variance to be close to zero, since there is no confounding-driven structure that needs to be absorbed by the random effect. The results suggest that the sample size has only a limited effect on the estimated variance component, whereas the number of features and the latent confounder dimension affect it more strongly. As the number of features increases, the estimated variance component decreases, which is consistent with the fact that under dense confounding the bias becomes smaller in high dimensions. In addition, empirical Bayes selects the smallest variance component in the unconfounded setting (latent dimension equal zero) and increases it as the latent confounder dimension grows, thereby allowing the random effect to absorb a larger fraction of the \(X\)-aligned confounding structure.
\begin{figure}[ht!]
    \centering
    \includegraphics[width=1\linewidth]{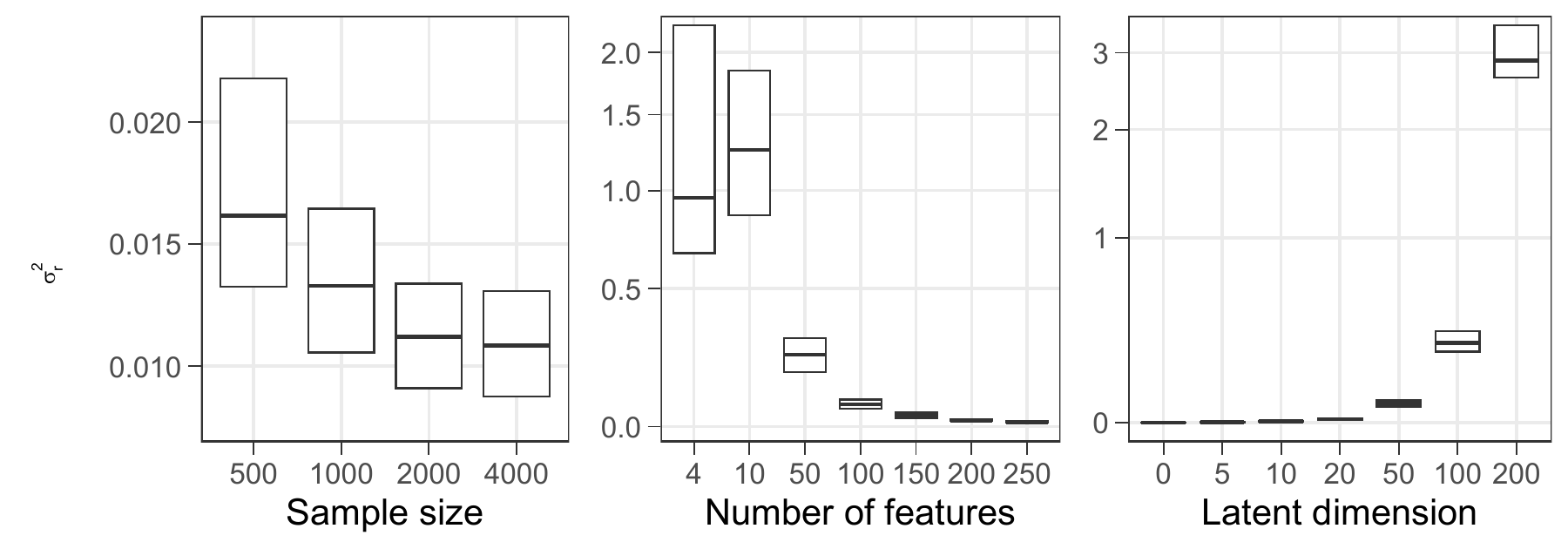}
   \caption{Regression setting.
    Boxplots of the estimated random-effect variance \(\sigma_r^2\),  displayed on a \(\log(1+x)\) scale on the y-axis. Larger values correspond to weaker shrinkage of the random effect.
    }
    \label{fig:pars_grid}
\end{figure}

\subsubsection{Classification setting}
\label{sec:classif_sim}

Figure~\ref{fig:mse_grid_2} reports the test mean squared error for the classification scenarios. The results are broadly consistent with those in the regression setting: although the differences are less pronounced than in regression, the proposed SpecBoost method still improves over standard boosting in most settings. Empirical Bayes also behaves reasonably in the classification setting, and the proposed rule for selecting the number of boosting iterations continues to track the oracle stopping rule well; see Figures~\ref{fig:iter_grid_2} and \ref{fig:pars_grid_2} in Appendix \ref{class_appd}.

\begin{figure}[ht!]
    \centering
    \includegraphics[width=1\linewidth]{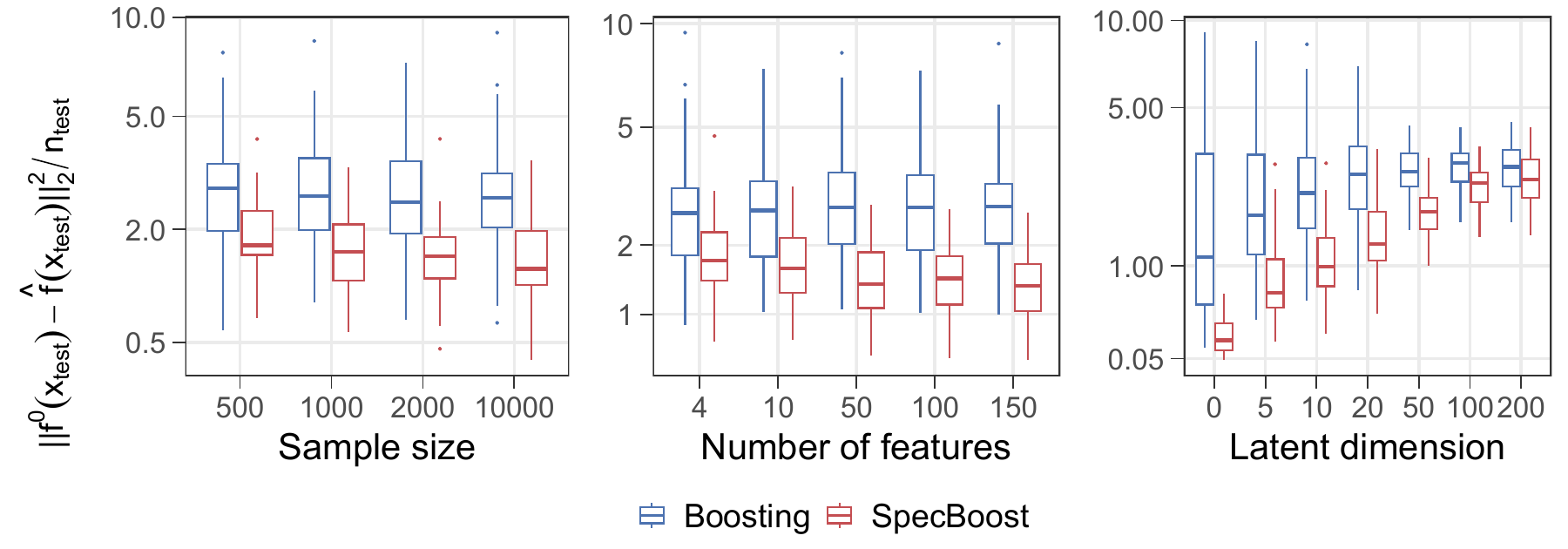}
    \caption{Classification setting.
    Boxplots of the mean squared error for the different methods across simulation scenarios,  displayed on a \(\log(1+x)\) scale on the y-axis. 
    }
    \label{fig:mse_grid_2}
\end{figure}

\subsection{Robustness of empirical-Bayes spectral boosting to violations of assumptions}
\label{sec:empirical bayes}

As discussed above, the spectral transformation underlying our method depends on a shrinkage parameter $\lambda_2$, which in the mixed-model formulation corresponds to the ratio of variance components $n\lambda_2 = \sigma_e^2 / \sigma_r^2$. In this section, we investigate the performance and robustness of the empirical Bayes estimate 
and that of a fixed rule proposed by \citet{cevid2020spectral}.

\paragraph{Fixed spectral rule.}
Following \citet{cevid2020spectral}, the regularization parameter $\lambda_2$ can be set as a function of the singular values of the design matrix $X$. In particular, the proposed rule selects $\lambda_2 = \frac{1}{n}d^2_{\lfloor \text{min}(n,p)/2 \rfloor}$. This choice is motivated by asymptotic arguments and implicitly assumes that the magnitude of the singular values of $X$ reflects the strength of the confounding.

\paragraph{Empirical Bayes.}
In contrast, the empirical Bayes approach estimates $n\lambda_2 = \sigma_e^2 / \sigma_r^2$ from the data by maximizing the marginal likelihood 
in Equation~\ref{eq:eb_objective}. Importantly, this criterion depends on the residual $Y - \hat{f}(X)$, and therefore potentially adapts to the observed magnitude of the confounding component. In particular, the empirical Bayes estimate reflects the scale of the bias term through its contribution to the residual, rather than relying solely on the geometry of $X$.

\subsubsection{Robustness to violations of the dense-confounding assumption}
Our approach relies on the dense-confounding assumption, namely that each latent confounder affects a large proportion of the observed covariates. We therefore investigate the robustness of the method when this assumption is progressively violated. We consider the default simulation setting with $n=1000$, $p=250$, and $q=20$, but vary the number $s$ of covariates affected by the confounders. Specifically, for each $s \in \{4,10,50,100,200,250\}$, we randomly select $s$ columns of the loading matrix $\Gamma$ and generate their entries independently from the same Gaussian distribution used in the baseline experiments, while setting the remaining $250-s$ columns of $\Gamma$ to zero. Consequently, only the selected covariates are influenced by the latent confounders.

Figure~\ref{fig:dense_robustness} shows that spectral boosting remains robust to violations of the dense-confounding assumption. As expected, the performance advantage gradually decreases as the confounding becomes sparser. Nevertheless, the proposed method remains competitive even when only a small subset of covariates is affected by the confounders. Moreover, under the strongest violations of the dense-confounding assumption, empirical-Bayes spectral boosting consistently outperforms both spectrally deconfounded random forests and the spectral boosting version using the fixed rule. This suggests that the empirical Bayes tuning strategy may provide additional robustness beyond the fixed asymptotic choice of the shrinkage parameter, although a more detailed theoretical understanding of this behavior remains an interesting direction for future work.
\begin{figure}[ht!]
    \centering
    \includegraphics[width=1\linewidth]{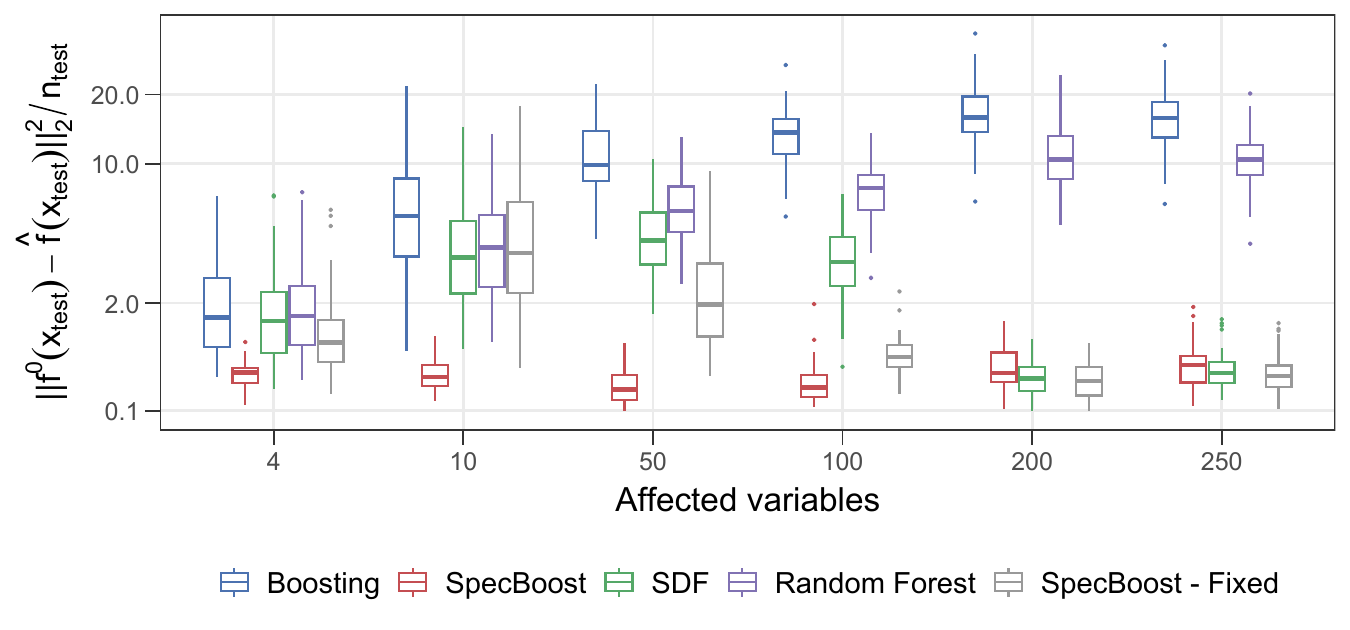}
    \caption{Robustness to violations of the dense-confounding assumption in the regression setting. The latent confounder affects only a subset of the $p=250$ covariates. The x-axis denotes the number of affected variables. Boxplots show the test mean squared error on a \(\log(1+x)\) scale.}
    \label{fig:dense_robustness}
\end{figure}

\subsubsection{Robustness to design misspecification}
We next investigate robustness under design misspecification. In the baseline setting, the design is generated as $X = H \Gamma + E$, where $E$ is isotropic Gaussian noise. To decouple the spectrum of $X$ from the strength of confounding, we now generate $X$ and $E$ as
\[
X = H \Gamma + E, 
\qquad
E = p^{2/3}\,  Z_E B_E / \sqrt{p} +  W_E,
\]
where $ Z_E \in \mathbb R^{n \times p}$, $B_E \in \mathbb R^{p \times p}$, and $\ W_{E} \in \mathbb R^{n \times p}$ have independent standard Gaussian entries. The term $Z_E B_E / \sqrt{p}$ has operator norm of order $O(\sqrt{p})$ when $n$ and $p$ grow at the same rate. 
Under Gaussian latent factors with fixed confounder dimension \(q\), the confounding component \(H\Gamma\) has operator norm of order \(O(\sqrt n\,\sqrt p)=O(p)\). The additional scaling by \(p^{2/3}\) inflates the noise component to order \(O(p^{7/6})\).

As a result, the spectrum of $X$ is driven primarily by the noise term, and the magnitude of its singular values is no longer informative about the strength of confounding. In contrast, the relevant quantity governing the bias is the projection of the confounding component onto the model class, which is reflected in the residual $Y - \hat f(X)$. Since empirical Bayes explicitly uses this residual to estimate $\lambda_2$, it adapts to the actual magnitude of the confounding present in the data. The fixed spectral rule, by contrast, depends only on the singular values of $X$ and therefore works less well when these singular values are dominated by variation unrelated to confounding. This leads to improved robustness of the empirical Bayes approach when the link between the spectrum of $X$ and the confounding strength is violated, as shown in Figure~\ref{fig:spiked_eb}.
\begin{figure}[ht!]
    \centering
    \includegraphics[width=1\linewidth]{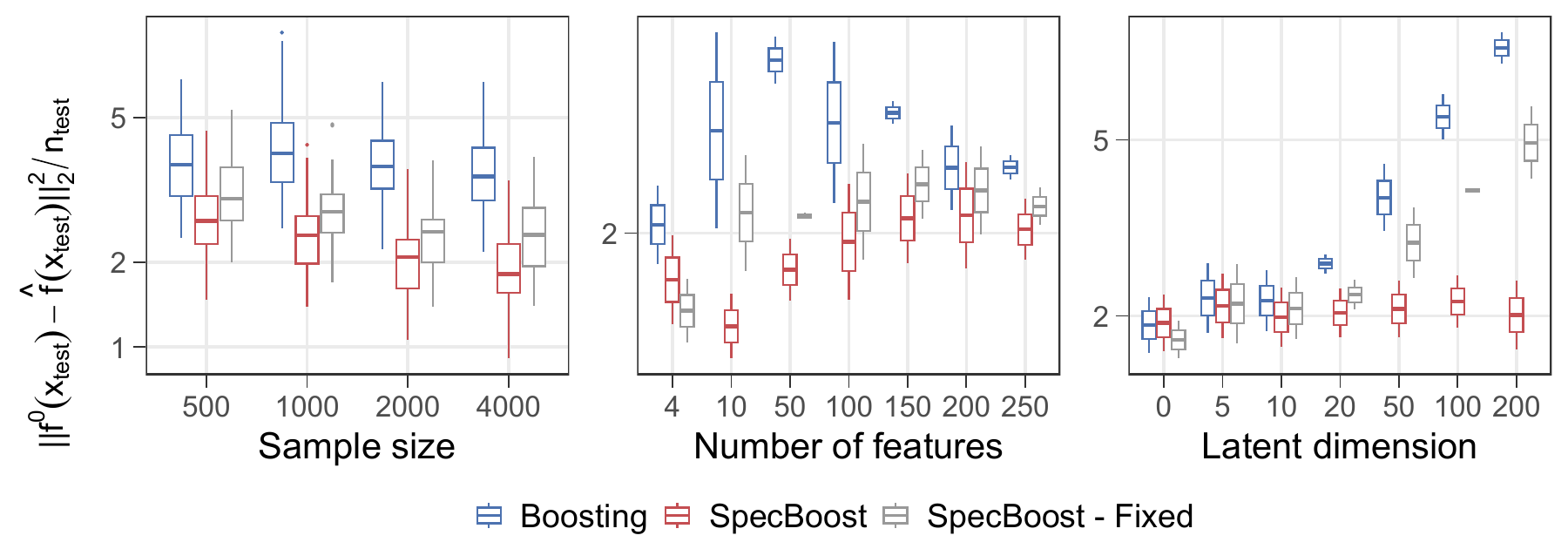}
    \caption{Design-misspecified regression setting.
Comparison of empirical-Bayes spectral tuning with the fixed rule of \citet{cevid2020spectral} when the spectrum of \(X\) is no longer dominated by the confounding component.
Boxplots show the test mean squared error on a \(\log(1+x)\) scale.}
    \label{fig:spiked_eb}
\end{figure}

\subsubsection{Well-specified setting}
We also consider a setting in which the model is correctly specified. In this regime, both the fixed rule and the empirical Bayes estimate yield comparable performance, as shown in Figure~\ref{fig:correct_eb}. The fixed rule performs slightly worse in low-dimensional regimes, where it tends to under-regularize. This behavior is expected, as the rule is derived under high-dimensional asymptotics.
\begin{figure}[ht!]
    \centering
    \includegraphics[width=1\linewidth]{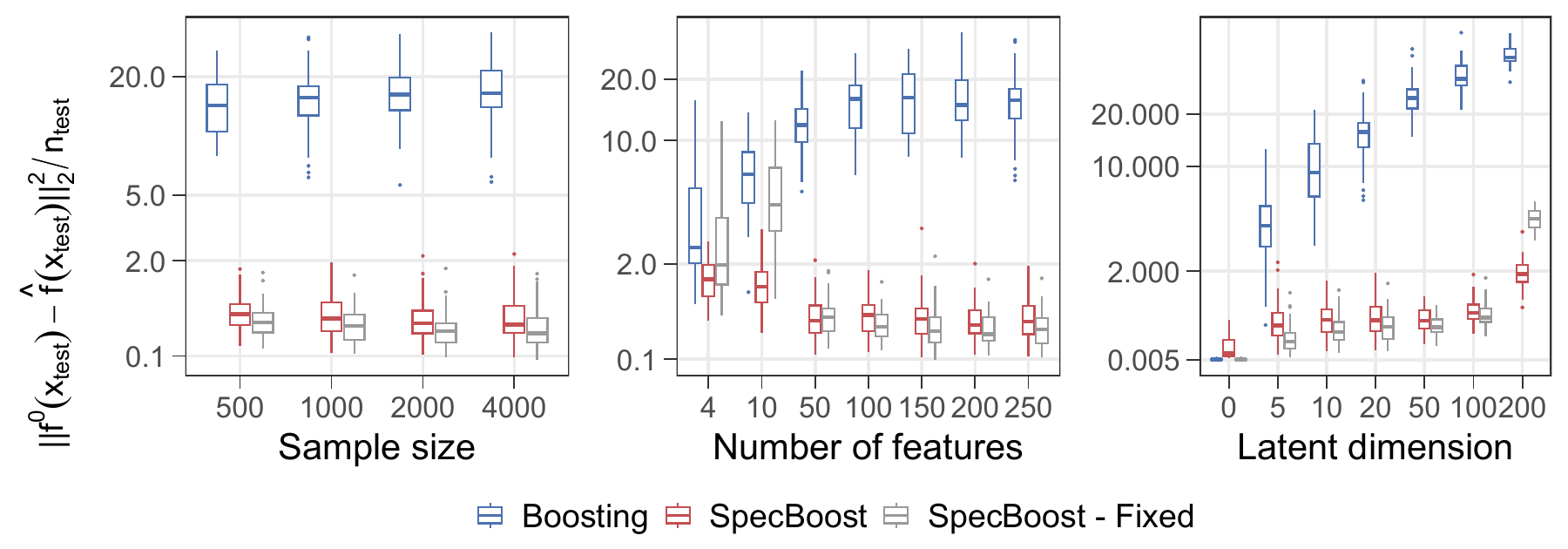}
    \caption{Well-specified regression setting.
    Comparison of empirical Bayes tuning of the spectral shrinkage with the fixed rule of \citet{cevid2020spectral}.
    Boxplots show the test mean squared error across simulation scenarios, displayed on a \(\log(1+x)\) scale.
    }
    \label{fig:correct_eb}
\end{figure}

\subsection{Nonlinear confounding}\label{sec:non-linear_sim}

We next consider the nonlinear confounding setting introduced in Section \ref{sec:non-linear_extension} and investigate the performance of the proposed kernel extension of spectral deconfounding in two experimental scenarios. Throughout this subsection, we use a Gaussian radial basis function (RBF) kernel. In the first experiment, all hyperparameters are estimated by empirical Bayes. In the second experiment, by contrast, we fix the random-effect variance using the rule of \citet{cevid2020spectral} and set the Gaussian kernel length scale via the median heuristic \citep{garreau2018median} since this scenario is more challenging: the leading kernel principal components represent a highly nonlinear structure, while in the first experiment, the confounding still lies in a linear subspace.

\subsubsection{Linear factor model with nonlinear outcome confounding}
We construct a family of confounding mechanisms that interpolates between linear and nonlinear regimes for the confounder $g(H)$. Let \(H_1\in\mathbb R^n\) denote the first column of \(H\), with entries
\(h_{i1}\stackrel{\mathrm{i.i.d.}}{\sim}\mathcal{N}(0,1)\), and define
\[
u_\alpha := g(H) 
=
\sqrt{1-\alpha}\,H_1
+
\sqrt{\alpha}\,(H_1^2-1),
\qquad \alpha\in[0,1].
\]
The two components in this construction are the first two Hermite polynomials, which are orthogonal under the standard Gaussian measure.
This orthogonality is useful because it separates first-order and second-order confounding effects. At \(\alpha=0\), the confounding is purely linear, $u_0 = H_1$;
so \(m(X)=\mathbb E[H_1\mid X]\) is linear in \(X\), and a ``correction" of the form \(X b\) is well matched. At \(\alpha=1\), the confounding is purely quadratic: $u_1 = H_1^2-1$.
In this case,
\[
m(X)=\mathbb E[H_1^2-1\mid X]
\]
is nonlinear in \(X\). Moreover, its linear projection vanishes. Indeed, for each \(X_j=\gamma_j H_1 + E_j\),
\[
\operatorname{Cov}(X_j, H_1^2-1)
=
\gamma_j \operatorname{Cov}(H_1,H_1^2-1)
+
\operatorname{Cov}(E_j,H_1^2-1)
=
0,
\]
since \(E_j\) is independent of \(H_1\) and \(\mathbb E[H_1(H_1^2-1)]=0\). Hence the best linear approximation to \(m(X)\) is zero, so linear spectral deconfounding is ineffective in this regime by construction. This experiment therefore isolates two important points: (i) the impact of confounding depends on how much of \(m(X)\) lies in the model class of the learner; (ii) a linear correction can only remove its best linear approximation \(X b\), whereas nonlinear learners may still absorb additional components of \(m(X)\).

We set \(q=20\), \(p=50\), and \(n=1000\) and generate $X, H$ and $f_0$ as above, and repeat the procedure $K=25$ times. The nonlinear confounding path is driven by the first latent factor \(H_1\), while the remaining latent coordinates enter through the factor model for \(X\). For each value of \(\alpha\), we construct \(u_\alpha\) as above and then rescale it so that its standard deviation is \(\tau=1.5\) times the standard deviation of the signal \(f_0(X)\). The response is generated as $
Y = f_0(X) + u_\alpha + \varepsilon,
$
with \(\varepsilon\) independent noise.
We evaluate five methods across values of \(\alpha \in \{0, 0.25, 0.5, 0.75, 1 \}\): standard boosting, boosting with a linear random effect (SpecBoost), boosting with an RBF kernel random effect (Kernel-SpecBoost), Random Forest \citep{wright2019package}, and the spectrally deconfounded random forest (SDF) \citep{Ulmer03042026}. To better interpret the results, we also measure how predictable the confounding signal \(u_\alpha\) is from \(X\) using linear regression, boosting, kernel ridge regression, and random forest. These diagnostics provide an empirical proxy for the projection of the induced bias \(m(X)\) onto the corresponding model classes.

Figure~\ref{fig:combined} summarizes the results. As expected, when \(\alpha=0\) the confounding is linear, and linear spectral deconfounding performs the best (left panel). As \(\alpha\) increases, however, the performance of the linear random-effect correction deteriorates, reflecting the fact that the linear nuisance \(X b\) becomes a progressively poorer approximation to the true bias \(m(X)\), as depicted by the right panel of Figure~\ref{fig:combined}. Importantly, this does not mean that linear spectral deconfounding becomes irrelevant immediately outside the exactly linear setting: even when \(m(X)\) is nonlinear, \(Xb\) remains its best linear approximation, and may still capture a non-negligible part of the confounding signal. Thus, linear spectral deconfounding can retain practical value beyond the idealized linear regime.

\begin{figure}[ht!]
    \centering
    \begin{subfigure}[t]{0.49\textwidth}
        \centering
        \includegraphics[width=\linewidth]{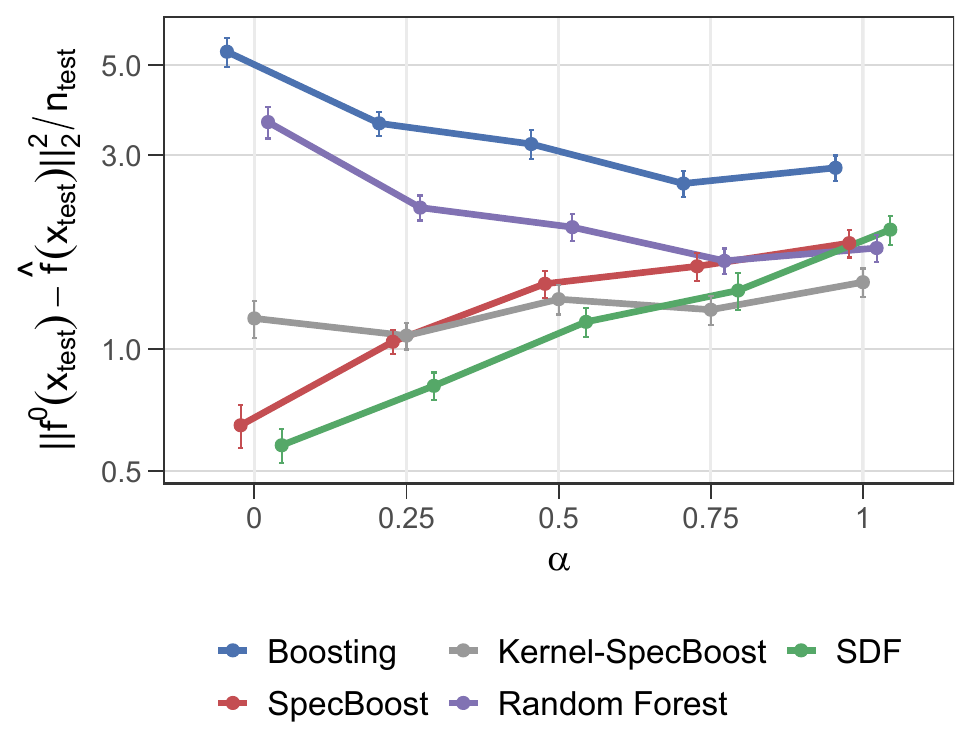}
        \label{fig:alpha_mse}
    \end{subfigure}
    \hfill
    \begin{subfigure}[t]{0.49\textwidth}
        \centering
        \includegraphics[width=\linewidth]{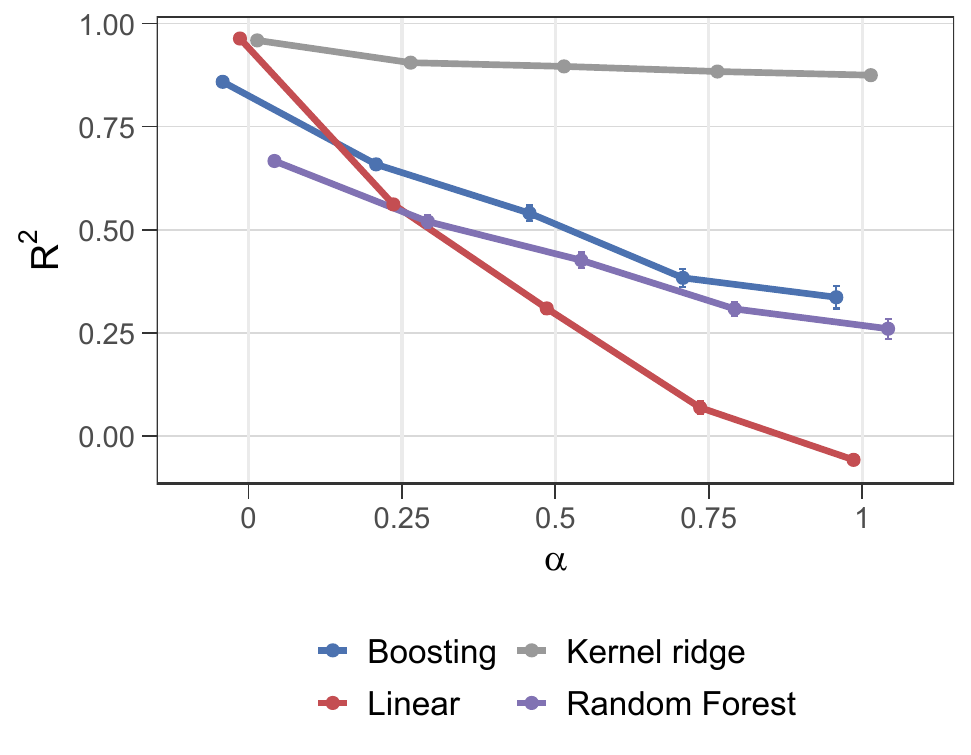}
        \label{fig:alpha_r2}
    \end{subfigure}
    \caption{Nonlinear outcome confounding.
    Performance as the confounding term becomes increasingly nonlinear. Vertical bars are twice the standard error of the mean performance. Left: test mean squared error. Right: predictability of the latent confounder, measured by the \(R^2\) from regressing \(u\) on \(X\).}
    \label{fig:combined}
\end{figure}

At \(\alpha=0\), the RBF random-effect model is also able to improve over standard boosting in this regime, consistent with the fact that its richer model class can represent the linear bias \(Xb\), although less efficiently. At intermediate values such as \(\alpha=0.5\), the two deconfounding methods perform comparably. At \(\alpha=1\), however, where the confounding has no linear component by construction, the RBF random-effect model outperforms the linear correction, implicit in both SpecBoost and SDF, whose performance becomes close to that of standard boosting. This suggests that a more expressive random effect can improve robustness by absorbing structured nonlinear confounding beyond the linear regime. 

The diagnostics plot in the right panel of Figure~\ref{fig:combined} helps explain these patterns. It reveals that strong nonlinear confounding may remain present even when its linear projection is weak. In particular, kernel ridge regression can often predict \(u_\alpha\) from \(X\) substantially better than linear regression, confirming that the induced bias \(m(X)\) is genuinely nonlinear. Yet boosted trees need not absorb all of this signal, suggesting that tree-based learners are relatively insensitive to certain dense nonlinear confounding mechanisms. A natural interpretation is that the confounding \(m(X)\) may have only a small projection onto the boosted-tree model class, even when it is large in a richer function class such as the RKHS induced by the RBF kernel.

\subsubsection{Nonlinear factor model and linear outcome confounding}
\label{sec:nonlinear_factor_model}

We now consider the setting in which the confounding arises from a nonlinear factor model in the covariates, while the confounder enters the outcome linearly. Specifically, we generate data as
\[
X = \phi(H) + E,
\qquad
Y = f_0(X) +  H\delta + \boldsymbol{\varepsilon},
\]
where ${H} \in \mathbb R^n$ and $\phi:\mathbb R\to\mathbb R^p$ is a nonlinear embedding applied row-wise. We construct $\phi(H)$ by first mapping $H$ to a one-dimensional spiral in $\mathbb R^2$,
$
Z = (H \cos H,\; H \sin H) \in \mathbb{R}^{n \times 2},
$
which is then linearly embedded row-wise into $\mathbb R^p$ via $ZB$, where $B \in \mathbb R^{2 \times p}$ is a random orthonormal embedding. This yields a data set that lies near a curved one-dimensional manifold in $\mathbb R^p$.
\begin{figure}[ht!]
    \centering
    \includegraphics[width=0.85\linewidth]{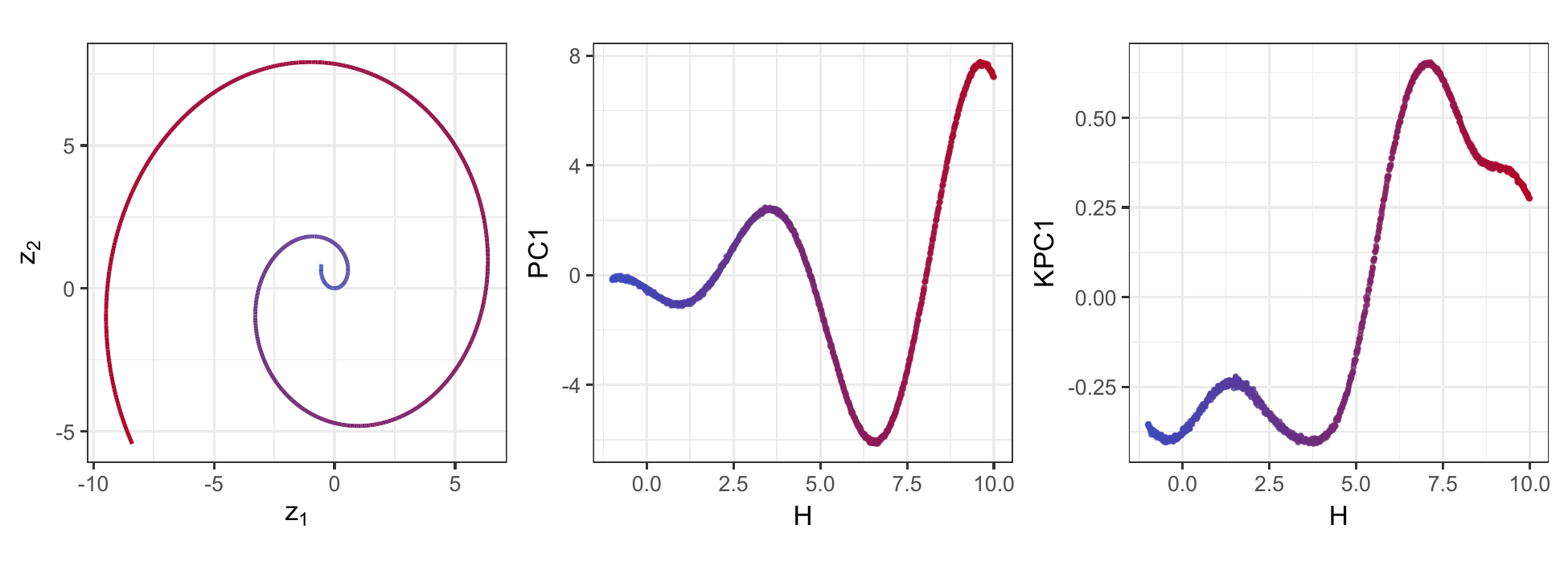}
    \caption{Nonlinear factor model.
Left: spiral manifold \(\phi(H)\) in \(\mathbb R^2\), colored by confounder \(H\).
Middle: First principal component. Right: First (RBF) kernel principal component.
}
    \label{fig:pc_kernel}
\end{figure}

Although the data lie in a two-dimensional linear subspace, the latent coordinate $H$ is not a linear function of $X$. As a result, linear PCs recover the subspace but fail to recover $H$. In contrast, kernel PCs capture nonlinear structure and recover coordinates that are strongly aligned with $H$ (Figure~\ref{fig:pc_kernel} and Figure~\ref{fig:spiral_results}).
Linear spectral methods can only attenuate linear projections of $m(X)$ and are therefore ineffective if the linear projection is small. In contrast, kernel spectral deconfounding adapts to the geometry of the data and can attenuate components aligned with the recovered nonlinear coordinates.  This aligns with the theoretical perspective developed in Section~\ref{sec:non-linear_extension}.

We generate training samples of size $n=1000$ and independent test samples of size $n_{\text{test}}=500$ and $p = 50$. The signal $f_0$ is constructed as in the previous section. The confounder $H$ is sampled uniformly and mapped through the spiral embedding described above. We repeat the experiment over $25$ independent replications. Figure~\ref{fig:spiral_results} shows that Kernel-SpecBoost achieves the lowest test error for estimating \(f_0\), outperforming both linear spectral methods and non-spectral baselines. This supports the need to account for nonlinear structure in \(X\): the diagnostic in the right panel shows that the first two kernel principal components explain substantially more variation in the latent coordinate \(H\) than the first two linear principal components.

\begin{figure}[ht!]
    \centering
    \includegraphics[width=0.79\linewidth]{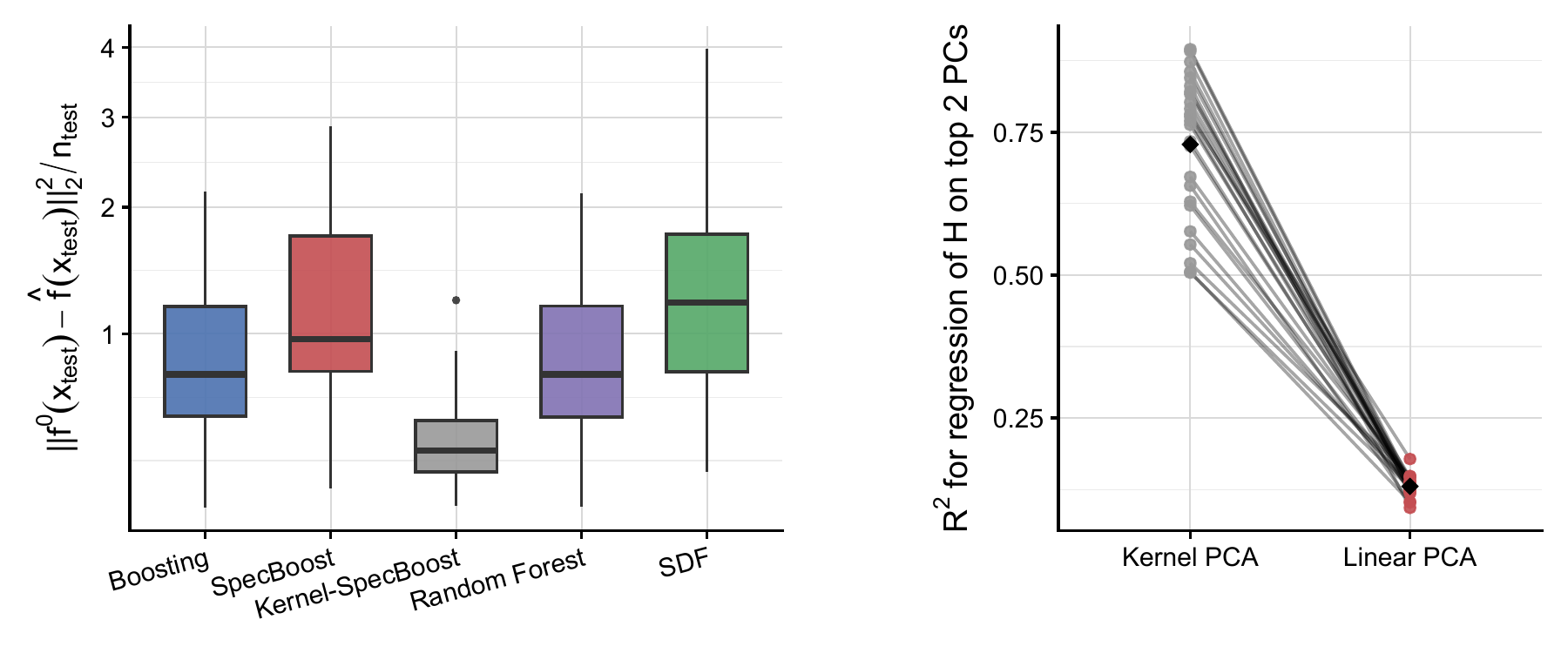}
    \caption{Regression setting with nonlinear factor model and linear outcome confounding.
Left: test mean squared error. Right: $\textit{R}^2$ of regressing the confounder $H$ on the first two principal components.
}
    \label{fig:spiral_results}
\end{figure}

\section{Real-World Application}
\label{sec:real_data}

We empirically validate our approach using the Boston housing data set \citep{blake1998uci}, obtained through the \texttt{mlbench} package \citep{leisch2024mlbench}, which comprises median house prices and neighborhood characteristics for 506 census tracts in the Boston metropolitan area. The data set includes geographic coordinates (latitude and longitude) for each tract. Since spatial location is known to have a substantial effect on housing prices even after accounting for observed predictor variables, fitting a predictive model without accounting for spatial variation introduces confounding whenever location is correlated with other covariates. In such cases, non-spatial features may act as proxies for location, absorbing spatial effects and thereby biasing estimated feature relationships. Our hypothesis is that spectral boosting will produce estimates more closely aligned with a spatial model despite not having access to explicit spatial information.

Figure~\ref{fig:two_panel}(a) displays the spatial distribution of house prices across the study region. Visual inspection reveals substantial spatial heterogeneity, with elevated prices concentrated in specific neighborhoods. Fitting models without coordinates therefore introduces confounding whenever this spatial effect is correlated with other features. Figure~\ref{fig:two_panel}(b) shows that the first principal component of the feature matrix exhibits a clear spatial pattern, consistent with dense confounding, where latent spatial effects load onto the leading principal components of the design matrix. This suggests that omitted spatial confounding may indeed bias estimators that ignore such structure.
\begin{figure}[ht!]
    \centering
    \begin{minipage}{0.49\linewidth}
        \centering
        \includegraphics[width=\linewidth]{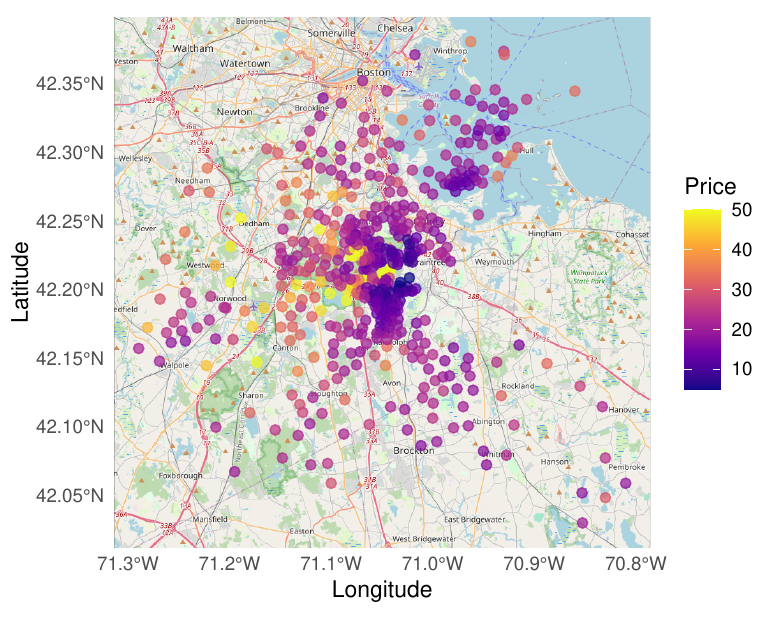}
    \end{minipage}\hfill
    \begin{minipage}{0.5\linewidth}
        \centering
        \includegraphics[width=\linewidth]{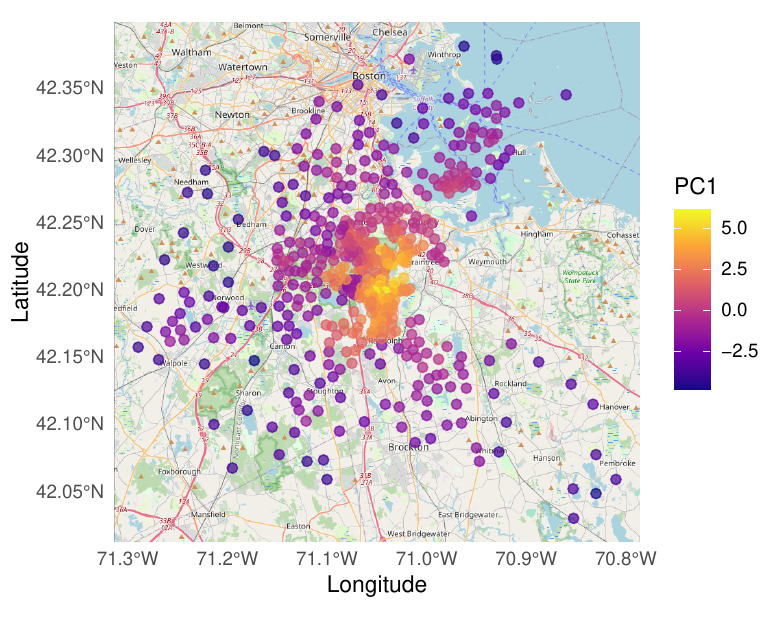}
    \end{minipage}
    \caption{
    Boston housing data.
Spatial distribution of median house prices (left) and of the first principal component of the non-spatial features (right).
    }
    \label{fig:two_panel}
\end{figure}

We run the following experiment to assess robustness to spatial confounding. First, we fit a \textit{spatial model} that includes geographic coordinates in a Gaussian process component with a Mat\'ern kernel, while using gradient boosting for the fixed effects, via the \texttt{GPBoost} package \citep{sigrist2022gaussian}. This serves as a reference model, since it explicitly accounts for spatial price variation. Second, we fit \textit{standard gradient boosting} using only the non-spatial features, deliberately omitting coordinates. Third, we fit our proposed \textit{spectral boosting} model, also excluding coordinates. In the latter two models, spatial location acts as a latent confounder. We then compare the methods by examining (i) the alignment of the fitted values with the reference model, and (ii) the stability of estimated feature effects via partial dependence plots.

Figure~\ref{fig:pdp_comparison} presents partial dependence plots for six features identified as important by permutation-based feature importance in the reference spatial boosting model. The plots show that spectral boosting produces feature effects that are closely aligned with those of the spatial boosting model, despite not having access to geographic coordinates. In contrast, standard gradient boosting exhibits substantial bias for several features. Furthermore, the mean squared error with respect to the predictions of the spatial boosting model is \(5.26\) for spectral boosting and \(45.40\) for standard boosting. Thus, predictions from the spectrally deconfounded model align well with the reference model, whereas this is not the case for the confounded boosting model.

\begin{figure}[ht!]
    \centering
    \includegraphics[width=1\linewidth]{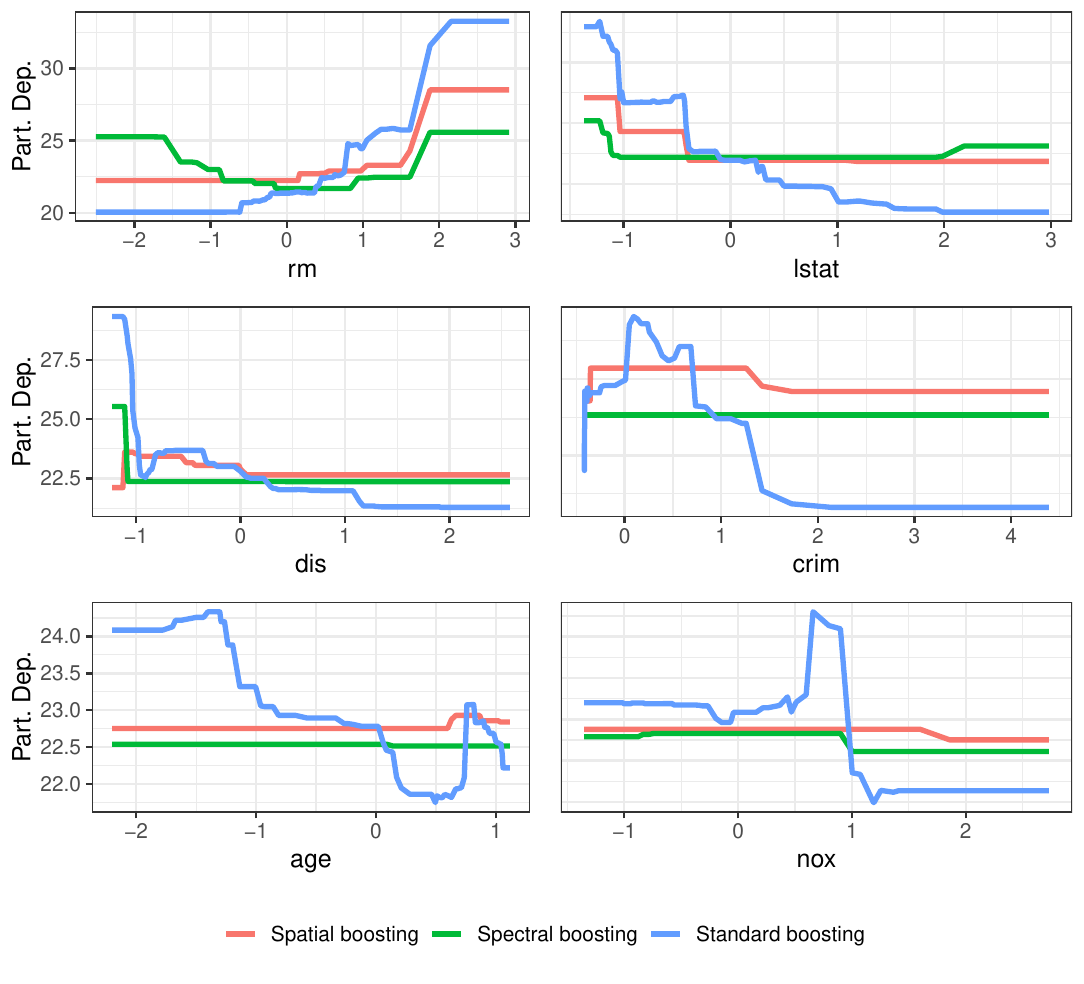}
   \caption{Boston housing data.
Partial dependence plots for the six features with highest permutation importance in the spatial boosting reference model.}
    \label{fig:pdp_comparison}
\end{figure}

The case of nitrogen oxide concentration (NOX) is particularly illustrative. Standard boosting suggests a non-monotone relationship: prices initially increase with NOX before dropping sharply at high concentrations. This contradicts the expected negative association between air pollution and housing values. The apparent paradox arises because NOX is spatially structured: it is correlated with both urban density and industrial land use. Without spatial adjustment, standard boosting interprets NOX as a proxy for location: moderate NOX levels coincide with desirable urban neighborhoods, whereas extreme values occur in industrial zones. The estimated effect therefore conflates the true effect of pollution with spatial price gradients. Our spectral deconfounding approach, by attenuating the component of feature variation aligned with spatial structure, recovers a monotonically decreasing relationship between NOX and price, consistent with the spatial model. Similar patterns emerge for other features, where the standard boosting model overestimates the effect of the feature on house prices because it also absorbs the effect of unobserved location: (i) overestimation of the price effect of \texttt{rm} (average number of rooms per dwelling) for large values; (ii) implausibly large marginal effects of \texttt{lstat} (percentage of lower-status population); (iii) a distorted nonlinear effect of \texttt{age}; and (iv) spikes in predicted house prices for certain values of \texttt{crim} (crime rate).

Furthermore, Figure~\ref{fig:correlation} compares the predicted random effects from the reference spatial boosting model and from our proposed spectral boosting model: the correlation coefficient is approximately $0.83$.
This shows that our method is able to capture a substantial part of the spatial variation without access to coordinates. Spectral boosting implicitly accounts for the spatial confounding through a random effect with a carefully chosen covariance structure aligned with the principal components of the design matrix, which, under dense confounding, align with the confounder (in this case spatial location). 

\begin{figure}[ht!]
    \centering
    \includegraphics[width=0.5\linewidth]{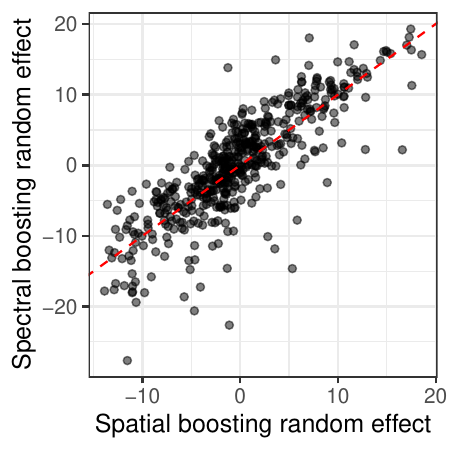}
    \caption{Boston housing data.
Predicted random effects from the spatial boosting reference model versus those from spectral boosting without access to geographic coordinates.}
    \label{fig:correlation}
\end{figure}

Overall, these results show that spectral boosting can substantially mitigate confounding from omitted spatial variables, yielding interpretable feature effects that are closely aligned with those obtained from explicit spatial modeling.

\section{Conclusion}
\label{sec:conclusion}

We introduce spectrally deconfounded gradient boosting, a boosting procedure designed to improve robustness to hidden dense confounding. The central idea is to embed spectral deconfounding directly into the loss: instead of fitting trees to raw residuals, the algorithm fits them to spectrally filtered residuals. This changes the optimization path by slowing learning in directions that are likely to be confounding-aligned.

Our analysis shows that the deconfounding effect does not come from the spectral loss or transformation alone. For continuous shrinkage transformations such as LAVA and Trim, the spectral weights remain positive at finite sample sizes, so a sufficiently flexible learner can eventually fit the confounding component. 
Rather, robustness arises from the interaction between spectral weighting and regularization, in particular in connection with early stopping. This creates an intermediate regime in which signal directions are learned while confounding-aligned directions remain largely unfitted. We also provide a mixed-model interpretation of the spectral loss, linking LAVA-type transformations to Gaussian random effects. This perspective yields an empirical-Bayes strategy for selecting the strength of spectral shrinkage and improves the suitability of cross-validation for hyperparameter tuning by incorporating the BLUP of the dense random-effect component. Moreover, it naturally extends the method to non-Gaussian likelihoods and to kernel-based deconfounding for nonlinear confounding structures.

Across simulations and a spatially confounded real-data example, spectral boosting improves recovery of the target signal relative to standard boosting while remaining computationally scalable. The results suggest that spectral losses provide a practical way to transfer ideas from linear spectral deconfounding and mixed-model adjustment to modern nonlinear learners. Important directions for future research include a sharper theory for tree-based base learners, principled stopping rules under distribution shift, and a deeper characterization of when kernel spectral transformations capture nonlinear confounding.

\section*{Acknowledgments}
The authors are grateful to Markus Ulmer, Cyrill Scheidegger, and Peter A. Whalley for useful discussions.

\bibliographystyle{plainnat}
\bibliography{references}

\newpage
\appendix
\section{Proofs for Section~\ref{sec:opt_view}}
\label{app:pathwise_unbiasedness}
\subsection{Proof of boosting recursion}
\begin{proof}
\label{app:proof_exact_bias}
For ordinary $L_2$ boosting, define
\[
r^{\mathrm{or}}_m := Y-\hat f^{\mathrm{or}}_m .
\]
Using OLS as base learner leads to
\[
r^{\mathrm{or}}_{m+1}
=
Y-\hat f^{\mathrm{or}}_{m+1}
=
(I-\nu P_X)r^{\mathrm{or}}_m .
\]
Since $r^{\mathrm{or}}_0=Y$, iteration gives
\[
r^{\mathrm{or}}_m=(I-\nu P_X)^mY,
\qquad
\hat f^{\mathrm{or}}_m
=
\bigl(I-(I-\nu P_X)^m\bigr)Y.
\]
Because $P_X$ is an orthogonal projector,
\[
(I-\nu P_X)^m
=
I-\bigl(1-(1-\nu)^m\bigr)P_X .
\]
Hence
\[
\hat f^{\mathrm{or}}_m
=
\bigl(1-(1-\nu)^m\bigr)P_XY.
\]
Taking conditional expectations and using
\[
P_Xf_0=f_0,\qquad P_Xg_0=g_0,\qquad \mathbb E[\varepsilon\mid X]=0,
\]
we obtain
\[
\mathbb E[\hat f^{\mathrm{or}}_m\mid X]
=
\bigl(1-(1-\nu)^m\bigr)(f_0+g_0),
\]
and therefore
\[
\mathbb E[\hat f^{\mathrm{or}}_m\mid X]-f_0
=
-(1-\nu)^m f_0
+
\bigl(1-(1-\nu)^m\bigr)g_0.
\]
For spectral OLS boosting, define
\[
r^{\mathrm{sp}}_m:=Y-\hat f^{\mathrm{sp}}_m .
\]
Then
\[
r^{\mathrm{sp}}_{m+1}
=
(I-\nu P_XW)r^{\mathrm{sp}}_m,
\qquad
r^{\mathrm{sp}}_m=(I-\nu P_XW)^mY.
\]
Since $P_X$ and $W$ are diagonal in the basis $U=(u_1,\dots,u_k)$,
\[
P_XWu_i=w_i u_i,\qquad i=1,\dots,k.
\]
Thus the recursion decouples coordinatewise:
\[
\langle r^{\mathrm{sp}}_m,u_i\rangle
=
(1-\nu w_i)^m\langle Y,u_i\rangle.
\]
Therefore
\[
\hat f^{\mathrm{sp}}_m
=
\sum_{i=1}^k
\bigl(1-(1-\nu w_i)^m\bigr)
\langle Y,u_i\rangle u_i.
\]
Using the expansions $f_0=\sum_i f_i u_i$, $g_0=\sum_i g_i u_i$, and
$\mathbb E[\varepsilon\mid X]=0$, we get
\[
\mathbb E[\hat f^{\mathrm{sp}}_m\mid X]-f_0
=
-\sum_{i=1}^k(1-\nu w_i)^m f_i u_i
+
\sum_{i=1}^k
\bigl(1-(1-\nu w_i)^m\bigr)g_i u_i.
\]
This proves the lemma.
\end{proof}

\subsection{Proof of Theorem~\ref{thm:pathwise_unbiasedness_main}}
\label{app:proof_theorem_main}
\begin{proof}
Let
\[
B:=
\mathbb E[\hat f^{\mathrm{sp}}_{m_p}\mid X]-f_0 .
\]
By the boosting recursion proved above,
\[
B
=
-\sum_{i=1}^k(1-\nu w_i)^{m_p}f_i u_i
+
\sum_{i=1}^k
\bigl(1-(1-\nu w_i)^{m_p}\bigr)g_i u_i .
\]
Split both sums into the top band $i\le q$ and the complementary band $i>q$.

For the signal on the complementary band, since $w_i\ge w_M$ for $i>q$,
\[
0\le (1-\nu w_i)^{m_p}
\le
(1-\nu w_M)^{m_p}
\le
\exp(-m_p\nu w_M)
\to 0.
\]
Hence
\[
\frac{1}{\sqrt n}
\left\|
\sum_{i>q}(1-\nu w_i)^{m_p}f_i u_i
\right\|_2
\le
(1-\nu w_M)^{m_p}\frac{\|f_0\|_2}{\sqrt n}
\to 0.
\]

For the confounding component on the top band, since $w_i\le w_H$ for $i\le q$,
\[
0\le
1-(1-\nu w_i)^{m_p}
\le
1-(1-\nu w_H)^{m_p}
\le
m_p\nu w_H
\to 0,
\]
where the last inequality follows from Bernoulli's inequality. Therefore
\[
\frac{1}{\sqrt n}
\left\|
\sum_{i\le q}
\bigl(1-(1-\nu w_i)^{m_p}\bigr)g_i u_i
\right\|_2
\le
m_p\nu w_H\frac{\|g_0\|_2}{\sqrt n}
\to 0.
\]

The two remaining terms are controlled directly by empirical spectral separation. Since
$0\le (1-\nu w_i)^{m_p}\le 1$,
\[
\frac{1}{\sqrt n}
\left\|
\sum_{i\le q}(1-\nu w_i)^{m_p}f_i u_i
\right\|_2
\le
\frac{\|\hat P_H f_0\|_2}{\sqrt n}
\to 0.
\]
Similarly, since $0\le 1-(1-\nu w_i)^{m_p}\le 1$,
\[
\frac{1}{\sqrt n}
\left\|
\sum_{i>q}
\bigl(1-(1-\nu w_i)^{m_p}\bigr)g_i u_i
\right\|_2
\le
\frac{\|(I-\hat P_H)g_0\|_2}{\sqrt n}
\to 0.
\]
Combining the four bounds gives
\[
\frac{\|B\|_2}{\sqrt n}\to 0,
\]
which proves the theorem.

For ordinary $L_2$ boosting with OLS base learners we have
\[
\mathbb E[\hat f^{\mathrm{or}}_{m_p}\mid X]-f_0
=
-(1-\nu)^{m_p}f_0
+
\bigl(1-(1-\nu)^{m_p}\bigr)g_0.
\]
Since $m_p\to\infty$ and $\nu\in(0,1)$, $(1-\nu)^{m_p}\to0$. Hence, by the reverse triangle
inequality,
\[
\frac{1}{\sqrt n}
\left\|
\mathbb E[\hat f^{\mathrm{or}}_{m_p}\mid X]-f_0
\right\|_2
\ge
\bigl(1-(1-\nu)^{m_p}\bigr)\frac{\|g_0\|_2}{\sqrt n}
-
(1-\nu)^{m_p}\frac{\|f_0\|_2}{\sqrt n}.
\]
The second term converges to zero, while the first is bounded away from zero under
our assumptions. Therefore the normalized bias cannot converge to zero.

Finally, if $w_H/w_M\to0$ and
\[
m_p\asymp\frac{1}{\nu\sqrt{w_Hw_M}},
\]
then
\[
m_p\nu w_M
\asymp
\sqrt{\frac{w_M}{w_H}}
\to\infty,
\qquad
m_p\nu w_H
\asymp
\sqrt{\frac{w_H}{w_M}}
\to0.
\]
Thus the learning rate condition of the theorem holds.
\end{proof}

\begin{remark}[Variance and full consistency]
\label{rem:variance_pathwise}
Theorem~\ref{thm:pathwise_unbiasedness_main} is a bias result. It shows that, under the spectral
separation and stopping conditions, the conditional mean of the boosted fit converges to the direct
sample-space signal $f=X\beta$. This does not by itself imply full consistency, since one must also
control the conditional variance.

For spectral OLS boosting, the fit is linear in the response:
\[
\hat f^{\mathrm{sp}}_m
=
B_mY,
\qquad
B_m
=
U\operatorname{diag}
\Bigl(
1-(1-\nu w_1)^m,\dots,1-(1-\nu w_k)^m
\Bigr)U^\top .
\]
Hence
\[
\operatorname{Var}(\hat f^{\mathrm{sp}}_m\mid X)
=
B_m\Sigma_\varepsilon B_m^\top,
\qquad
\Sigma_\varepsilon:=\operatorname{Var}(\varepsilon\mid X).
\]
If $\Sigma_\varepsilon\preceq\sigma^2 I_n$, then
\[
\frac{1}{n}
\operatorname{tr}\!\left(
\operatorname{Var}(\hat f^{\mathrm{sp}}_m\mid X)
\right)
\le
\frac{\sigma^2}{n}
\sum_{i=1}^k
\Bigl(1-(1-\nu w_i)^m\Bigr)^2.
\]
This motivates the effective learned dimension
\[
d_{\mathrm{eff}}(B_m)
:=
\operatorname{tr}(B_m^2)
=
\sum_{i=1}^k
\Bigl(1-(1-\nu w_i)^m\Bigr)^2.
\]
Under the stopping condition of Theorem~\ref{thm:pathwise_unbiasedness_main},
$m_p\nu w_M\to\infty$, all directions outside the top confounding band are eventually learned almost
fully. Consequently,
\[
d_{\mathrm{eff}}(B_{m_p})\gtrsim k-q.
\]
Thus vanishing prediction variance is ensured by an additional condition such as
\[
\frac{d_{\mathrm{eff}}(B_{m_p})}{n}\to0,
\qquad\text{or more conservatively}\qquad
\frac{k}{n}\to0.
\]
This is substantially stronger than the bias assumptions and is generally restrictive when $p$ is of the
same order as, or larger than, $n$.
In the latter case, a natural route
would be componentwise boosting, whose effective learned dimension is tied more directly to the
number of selected coordinates. However, in that case the base learner is no longer a linear operator,
so the explicit spectral recursion and the present analysis no longer
apply directly.
\end{remark}

\section{Factor-model verification of the pathwise assumptions}
\label{app:factor_model_justification}

This appendix verifies the assumptions of
Theorem~\ref{thm:pathwise_unbiasedness_main} under the dense-confounding factor model
\[
X=H\Gamma+E,
\qquad
Y=X\beta_0+H\delta+\eta.
\]
Here $H\in\mathbb R^{n\times q}$ contains the latent confounders,
$\Gamma\in\mathbb R^{q\times p}$ is the loading matrix, and $E$ is the unconfounded component.
The number of confounders $q$ is fixed, and all limits are along $p\to\infty$, with
$n=n_p \to \infty$.

\subsection{Verification of Assumption~\ref{ass:reduced_model_energy_main}}

Assumption~\ref{ass:reduced_model_energy_main} requires the reduced model
\[
Y=f_0+g_0+\varepsilon,
\qquad
f_0=X\beta_0,
\qquad
g_0=Xb_0,
\]
with $\mathbb E[\varepsilon\mid X]=0$ and bounded empirical $L^2$ norms of $f_0$ and $g_0$.

The following lemma verifies the required properties under dense confounding.

\begin{lemma}[Reduced model and bounded empirical energy]
\label{lem:A1_factor_app}
Assume that the rows $h_i$ of $H$ are iid mean-zero sub-Gaussian with
\[
\operatorname{Cov}(h_i)=I_q,
\]
and that the rows $e_i$ of $E$ are iid mean-zero sub-Gaussian, independent of $H$, with covariance
$\Sigma_E$ satisfying
\[
0<c_E\le \lambda_{\min}(\Sigma_E)
\le
\lambda_{\max}(\Sigma_E)\le C_E<\infty.
\]
Assume also
\[
d_q(\Gamma)\gtrsim \sqrt p,
\qquad
\|\delta\|_2=O(1),
\qquad
\frac{\|X\beta_0\|_2}{\sqrt n}=O_p(1).
\]
Let
\[
\Sigma_X:=\operatorname{Cov}(x_i)=\Gamma^\top\Gamma+\Sigma_E,
\qquad
b_0:=\Sigma_X^{-1}\Gamma^\top\delta,
\]
and define
\[
r:=H\delta-Xb_0.
\]
Then
\[
\frac{\|r\|_2}{\sqrt n}=o_p(1),
\qquad
\frac{\|Xb_0\|_2}{\sqrt n}=O_p(1).
\]
Consequently,
\[
Y=X\beta_0+Xb_0+r+\eta
\]
is an $o_p(1)$ perturbation, in normalized empirical $L^2$ norm, of the reduced model
\[
Y=f_0+g_0+\varepsilon,
\qquad
f_0=X\beta_0,
\qquad
g_0=Xb_0.
\]
If, in addition, $(H,E)$ are jointly Gaussian and $\mathbb E[\eta\mid X]=0$, then
$\mathbb E[r+\eta\mid X]=0$.
\end{lemma}

\subsection{Verification of Assumption~\ref{ass:empirical_separation_main}}

Assumption~\ref{ass:empirical_separation_main} requires
\[
\frac{\|(I-\hat P_H)g_0\|_2}{\sqrt n}\to0,
\qquad
\frac{\|\hat P_H f_0\|_2}{\sqrt n}\to0.
\]
The first condition says that the fitted confounding term lies in the top empirical spectral
subspace. The second condition says that the direct signal does not.

\begin{lemma}[Empirical spectral separation]
\label{lem:A2_factor_app}
Assume the conditions of Lemma~\ref{lem:A1_factor_app}. Let $P_H$ be the orthogonal projector
onto $\operatorname{col}(H)$, and let
\[
\hat P_H=\sum_{i=1}^q u_i u_i^\top
\]
be the projector onto the top $q$ left singular vectors of $X$.

Assume that the top empirical factor space is consistently estimated. Equivalently, assume that
\[
\|\hat P_H-P_H\|_{\mathrm{op}}=o_p(1).
\]
This condition is implied by standard strong-factor PCA results; for example,
\citet{bai2003inferential} gives factor recovery up to rotation, and Wedin's sin-theta theorem
or Davis--Kahan perturbation theory transfers this to projector convergence
\citep{yu2015useful}.

Assume also the structural signal-separation condition
\[
\frac{\|\hat P_HX\beta_0\|_2}{\sqrt n}=o_p(1).
\]
Then, with $f_0=X\beta_0$ and $g_0=Xb_0$,
\[
\frac{\|(I-\hat P_H)g_0\|_2}{\sqrt n}=o_p(1),
\qquad
\frac{\|\hat P_Hf_0\|_2}{\sqrt n}=o_p(1).
\]
\end{lemma}

\subsection{Verification of Assumption~\ref{ass:weight_separation_main}}

Assumption~\ref{ass:weight_separation_main} requires
\[
w_H:=\max_{i\le q}w_i\to0,
\qquad
w_M:=\min_{q<i\le k}w_i
\]
to stay bounded away from zero. We verify this for LAVA-type weights
\[
w_i=\frac{n\lambda_2}{n\lambda_2+d_i(X)^2},
\]
where $d_i(X)$ denotes the $i$th singular value of $X$.

\begin{lemma}[LAVA weight separation]
\label{lem:A3_factor_app}
Assume that the rows $h_i$ of $H$ are iid mean-zero sub-Gaussian with
$\operatorname{Cov}(h_i)=I_q$, and that the rows $e_i$ of $E$ are iid mean-zero sub-Gaussian,
independent of $H$, with
\[
\lambda_{\max}(\Sigma_E)\le C_E<\infty.
\]
Assume
\[
d_q(\Gamma)\gtrsim \sqrt p,
\qquad
\min(n,p)\to\infty.
\]
Let the LAVA parameter be chosen at the bulk singular-value scale:
\[
n\lambda_2\asymp(\sqrt n+\sqrt p)^2.
\]
Then
\[
w_H=o_p(1),
\]
and there exists $c_0>0$ such that
\[
\mathbb P(w_M\ge c_0)\to1.
\]
\end{lemma}

\subsection{Factor-model corollary}

Combining the preceding lemmas gives the following stochastic version of the pathwise theorem.

\begin{corollary}[Dense factor model implies spectral pathwise deconfounding]
\label{cor:factor_model_verification}
Suppose the assumptions of Lemmas~\ref{lem:A1_factor_app}--\ref{lem:A3_factor_app} hold. Assume
also joint Gaussianity of $(H,E)$ and $\mathbb E[\eta\mid X]=0$, so that the reduced model has exact
conditional mean zero. Let the LAVA weights satisfy
\[
n\lambda_2\asymp(\sqrt n+\sqrt p)^2.
\]
If the stopping sequence satisfies
\[
m_p\to\infty,
\qquad
\frac{m_p\nu}{\min(n,p)}\to0,
\]
then
\[
m_p\nu w_M\to\infty
\quad\text{in probability},
\qquad
m_p\nu w_H\to0
\quad\text{in probability}.
\]
Moreover,
\[
\frac{1}{\sqrt n}
\left\|
\mathbb E[\hat f^{\mathrm{sp}}_{m_p}\mid X]-X\beta_0
\right\|_2
=o_p(1).
\]
\end{corollary}

The deterministic proof of Theorem~\ref{thm:pathwise_unbiasedness_main} is pathwise. Under the
factor model, the same deterministic upper bounds hold with $o_p(1)$ terms. Since finite sums of
$o_p(1)$ terms are $o_p(1)$, the conclusion of the theorem transfers to convergence in probability.

\begin{remark}[Non-Gaussian factor models]
If joint Gaussianity is not assumed, the projection residual $r=H\delta-Xb_0$ need not satisfy
$\mathbb E[r\mid X]=0$ exactly. Lemma~\ref{lem:A1_factor_app} still gives
$\|r\|_2/\sqrt n=o_p(1)$, so the same conclusion holds for the approximate reduced model after
adding this $o_p(1)$ contribution to the normalized bias.
\end{remark}

\subsection{Proof of Lemma~\ref{lem:A1_factor_app}}
\begin{proof}
This proof follows closely the one presented in \citet{10.1145/3711116}.

For one observation,
\[
r_i=h_i^\top\delta-x_i^\top b_0.
\]
Since $b_0$ is the population linear projection coefficient of $h_i^\top\delta$ onto $x_i$,
\[
\mathbb E[r_i^2]
=
\operatorname{Var}(h_i^\top\delta)
-
\operatorname{Cov}(h_i^\top\delta,x_i)
\Sigma_X^{-1}
\operatorname{Cov}(x_i,h_i^\top\delta).
\]
Using $\operatorname{Cov}(h_i)=I_q$ and $x_i=\Gamma^\top h_i+e_i$,
\[
\operatorname{Var}(h_i^\top\delta)=\|\delta\|_2^2,
\qquad
\operatorname{Cov}(x_i,h_i^\top\delta)=\Gamma^\top\delta.
\]
Therefore
\[
\mathbb E[r_i^2]
=
\delta^\top
\left[
I_q-\Gamma(\Gamma^\top\Gamma+\Sigma_E)^{-1}\Gamma^\top
\right]\delta.
\]
By the Woodbury identity,
\[
I_q-\Gamma(\Gamma^\top\Gamma+\Sigma_E)^{-1}\Gamma^\top
=
\left(I_q+\Gamma\Sigma_E^{-1}\Gamma^\top\right)^{-1}.
\]
Thus
\[
\mathbb E[r_i^2]
=
\delta^\top
\left(I_q+\Gamma\Sigma_E^{-1}\Gamma^\top\right)^{-1}
\delta.
\]
Since
\[
\lambda_{\min}(\Gamma\Sigma_E^{-1}\Gamma^\top)
\ge
\lambda_{\min}(\Sigma_E^{-1})d_q(\Gamma)^2
\gtrsim p,
\]
we get
\[
\mathbb E[r_i^2]\lesssim \frac{\|\delta\|_2^2}{p}.
\]
Hence
\[
\mathbb E\left[\frac1n\|r\|_2^2\right]
=
\mathbb E[r_1^2]
\lesssim
\frac{\|\delta\|_2^2}{p}.
\]
Because $\|\delta\|_2=O(1)$, Markov's inequality gives
\[
\frac1n\|r\|_2^2=o_p(1),
\qquad
\frac{\|r\|_2}{\sqrt n}=o_p(1).
\]
Next,
\[
\frac{\|Xb_0\|_2}{\sqrt n}
\le
\frac{\|H\delta\|_2}{\sqrt n}
+
\frac{\|r\|_2}{\sqrt n}.
\]
Moreover,
\[
\frac1n\|H\delta\|_2^2
=
\delta^\top\left(\frac1nH^\top H\right)\delta.
\]
Since $H^\top H/n\to I_q$ in probability and $\|\delta\|_2=O(1)$,
\[
\frac{\|H\delta\|_2}{\sqrt n}=O_p(1).
\]
Therefore
\[
\frac{\|Xb_0\|_2}{\sqrt n}=O_p(1).
\]

If $(H,E)$ are jointly Gaussian, then $r$ is uncorrelated with $X$ by construction of the population
projection. Gaussianity implies independence, and hence $\mathbb E[r\mid X]=0$.
\end{proof}

\subsection{Proof of Lemma~\ref{lem:A2_factor_app}}
\begin{proof}
By Lemma~\ref{lem:A1_factor_app},
\[
Xb_0=H\delta-r,
\qquad
\frac{\|r\|_2}{\sqrt n}=o_p(1).
\]
Therefore
\[
\frac{\|(I-\hat P_H)Xb_0\|_2}{\sqrt n}
\le
\frac{\|(I-\hat P_H)H\delta\|_2}{\sqrt n}
+
o_p(1).
\]
Since $H\delta\in\operatorname{col}(H)$,
\[
P_HH\delta=H\delta.
\]
Thus
\[
(I-\hat P_H)H\delta
=
(P_H-\hat P_H)H\delta.
\]
Therefore
\[
\frac{\|(I-\hat P_H)H\delta\|_2}{\sqrt n}
\le
\|\hat P_H-P_H\|_{\mathrm{op}}
\frac{\|H\delta\|_2}{\sqrt n}.
\]
By assumption,
\[
\|\hat P_H-P_H\|_{\mathrm{op}}=o_p(1).
\]
By the proof of Lemma~\ref{lem:A1_factor_app},
\[
\frac{\|H\delta\|_2}{\sqrt n}=O_p(1).
\]
Hence
\[
\frac{\|(I-\hat P_H)Xb_0\|_2}{\sqrt n}=o_p(1).
\]
The second statement,
\[
\frac{\|\hat P_HX\beta_0\|_2}{\sqrt n}=o_p(1),
\]
is exactly the assumed signal-separation condition.
\end{proof}

\subsection{Proof of Lemma~\ref{lem:A3_factor_app}}
\begin{proof}
We first lower-bound the top singular values of $X$. Weyl's inequality for singular values gives
\[
d_q(X)
=
d_q(H\Gamma+E)
\ge
d_q(H\Gamma)-d_1(E).
\]
Moreover,
\[
d_q(H\Gamma)\ge d_q(H)d_q(\Gamma).
\]
Since the rows of $H$ are iid sub-Gaussian with covariance $I_q$ and $q$ is fixed,
\[
d_q(H)\asymp_p \sqrt n.
\]
By assumption,
\[
d_q(\Gamma)\gtrsim \sqrt p.
\]
Finally, by standard sub-Gaussian operator-norm bounds,
\[
d_1(E)=O_p(\sqrt n+\sqrt p)
\]
\citep[Chapter~5]{vershynin2020high}. Since $\min(n,p)\to\infty$,
\[
\sqrt{np}\gg \sqrt n+\sqrt p.
\]
Therefore
\[
d_i(X)^2\gtrsim_p np,
\qquad i=1,\dots,q.
\]
We now control the first non-spiked singular value. Again by Weyl's inequality,
\[
d_{q+1}(X)
=
d_{q+1}(H\Gamma+E)
\le
d_{q+1}(H\Gamma)+d_1(E).
\]
Since $\operatorname{rank}(H\Gamma)\le q$,
\[
d_{q+1}(H\Gamma)=0.
\]
Hence
\[
d_{q+1}(X)\le d_1(E)=O_p(\sqrt n+\sqrt p).
\]
Thus
\[
d_{q+1}(X)^2
=
O_p\bigl((\sqrt n+\sqrt p)^2\bigr).
\]
With $
n\lambda_2\asymp(\sqrt n+\sqrt p)^2,
$
the top weights satisfy
\[
w_H
=
\max_{i\le q}
\frac{n\lambda_2}{n\lambda_2+d_i(X)^2}
=
O_p\left(\frac{(\sqrt n+\sqrt p)^2}{np}\right).
\]
Since
\[
\frac{(\sqrt n+\sqrt p)^2}{np}
\lesssim
\frac{1}{\min(n,p)},
\]
we get
\[
w_H=o_p(1).
\]
For $i>q$, the bulk bound gives, with probability tending to one,
\[
d_i(X)^2\le C(\sqrt n+\sqrt p)^2.
\]
Therefore, with probability tending to one,
\[
w_i
\ge
\frac{c(\sqrt n+\sqrt p)^2}
{c(\sqrt n+\sqrt p)^2+C(\sqrt n+\sqrt p)^2}
\ge c_0>0.
\]
Taking the minimum over $i>q$ gives
\[
\mathbb P(w_M\ge c_0)\to1.
\]
\end{proof}
\vspace{-1.35cm}
\section{Additional Simulation Results}
\subsection{Classification}\label{class_appd}
\begin{figure}[h]
    \centering
    \includegraphics[width=0.9\linewidth]{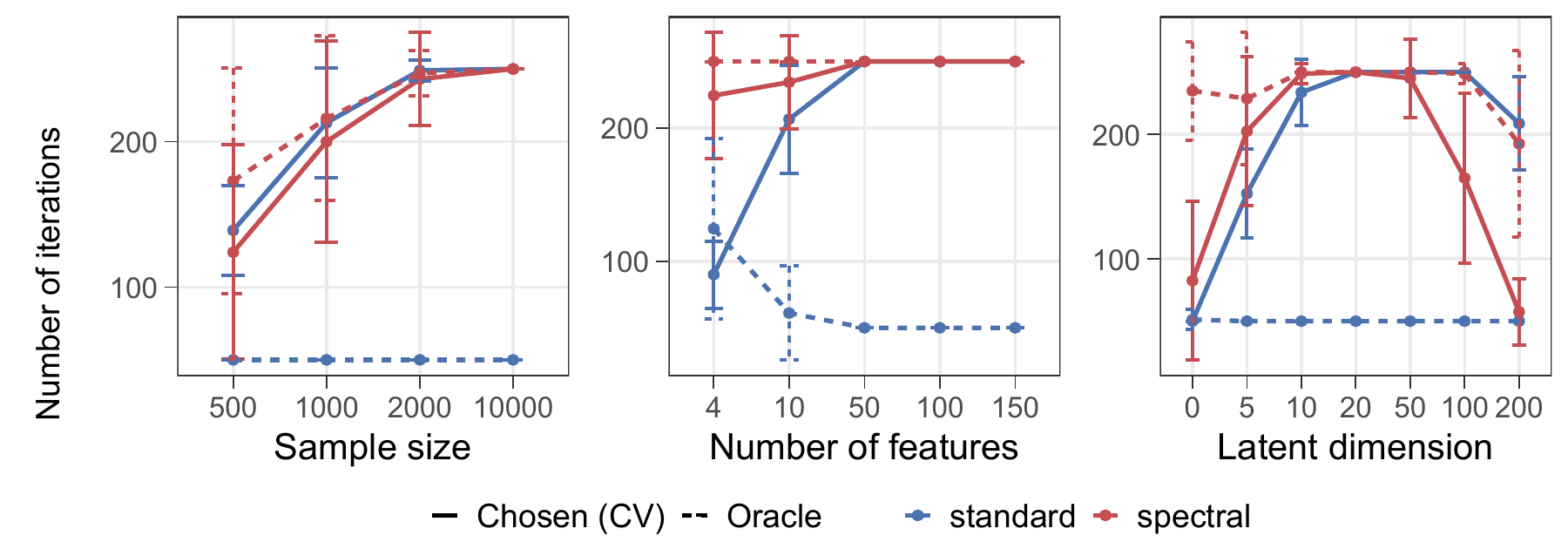}
   \caption{Classification setting.
     Number of boosting iterations selected by the proposed cross-validation rule
    (solid lines) and by the oracle rule with access to the true target function \(f_0\) (dashed lines).
    }
    \label{fig:iter_grid_2}
\end{figure}

\begin{figure}[h]
    \centering
    \includegraphics[width=0.9\linewidth]{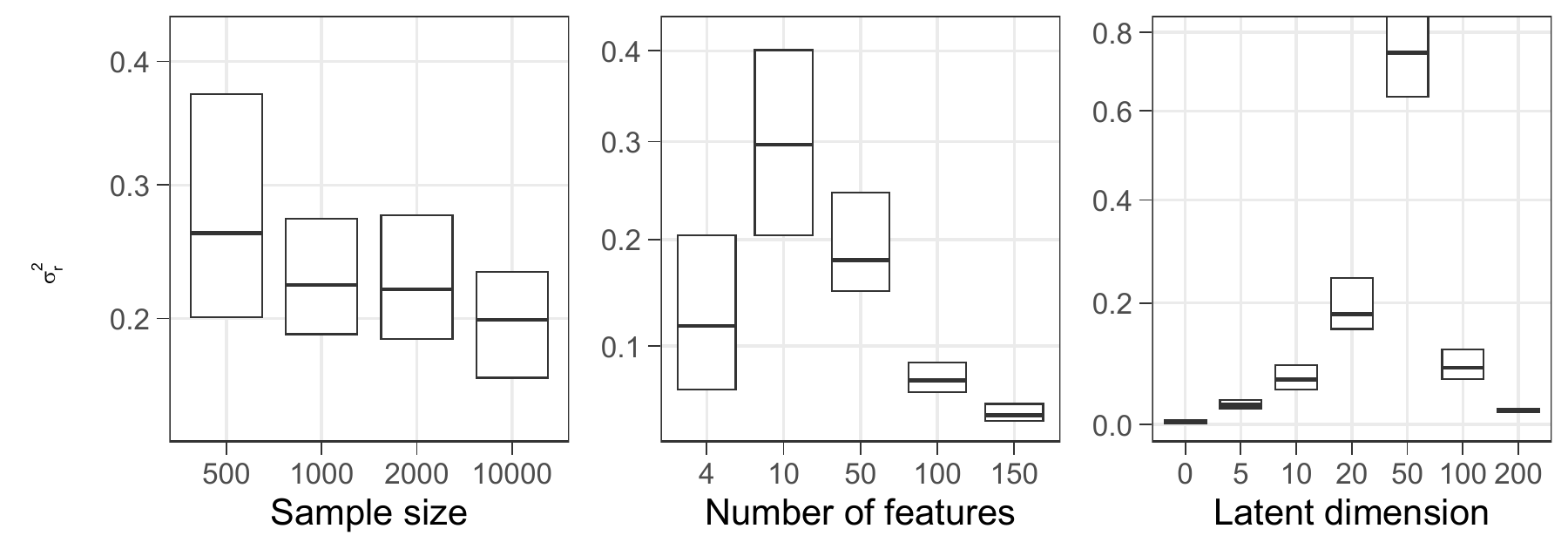}
    \caption{Classification setting.
    Boxplots of the estimated random-effect variance \(\sigma_r^2\),  displayed on a \(\log(1+x)\) scale on the y-axis. 
    }
    \label{fig:pars_grid_2}
\end{figure}

\begin{figure}[h]
    \centering
    \includegraphics[width=0.9\linewidth]{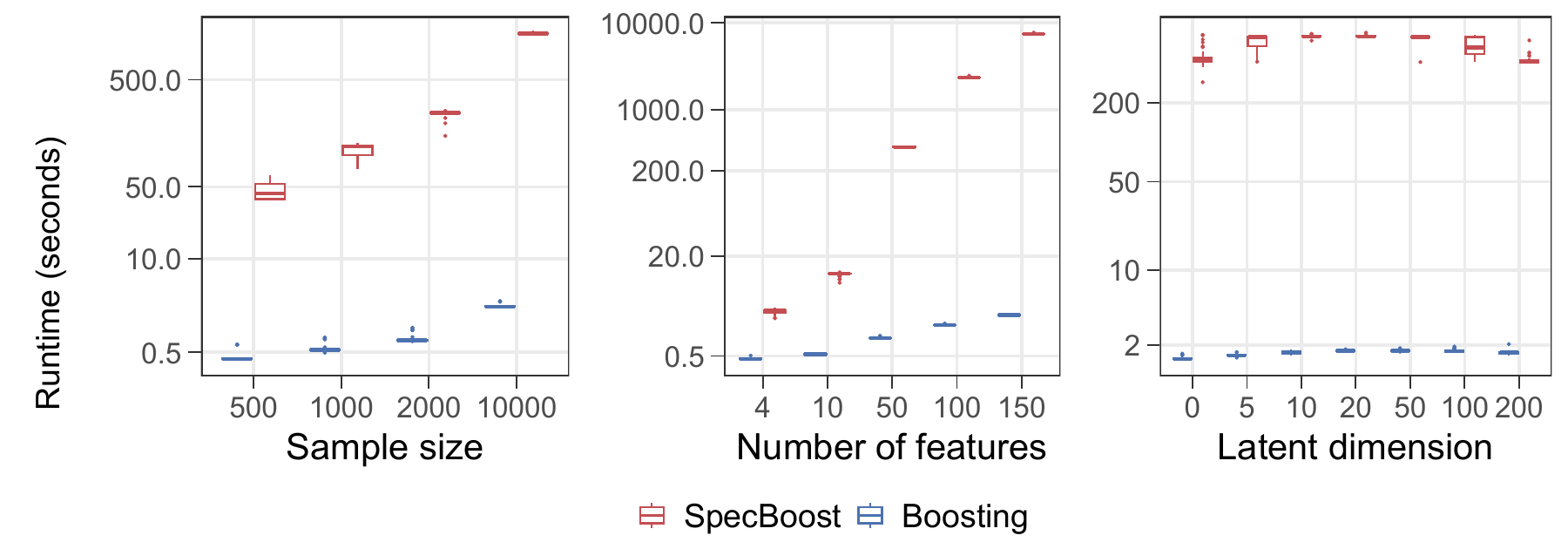}
    \caption{
    Classification setting. Boxplots of the runtime for the different methods across simulation scenarios,  displayed on a \(\log(1+x)\) scale on the y-axis. 
    }
\end{figure}

\subsection{Regression - No confounding scenario}
\label{app:no_confounding_regression}
In this experiment the confounder's contribution to the outcome is set to zero, i.e. $\delta = 0$.
\begin{figure}[h!]
    \centering
    \includegraphics[width=.9\linewidth]{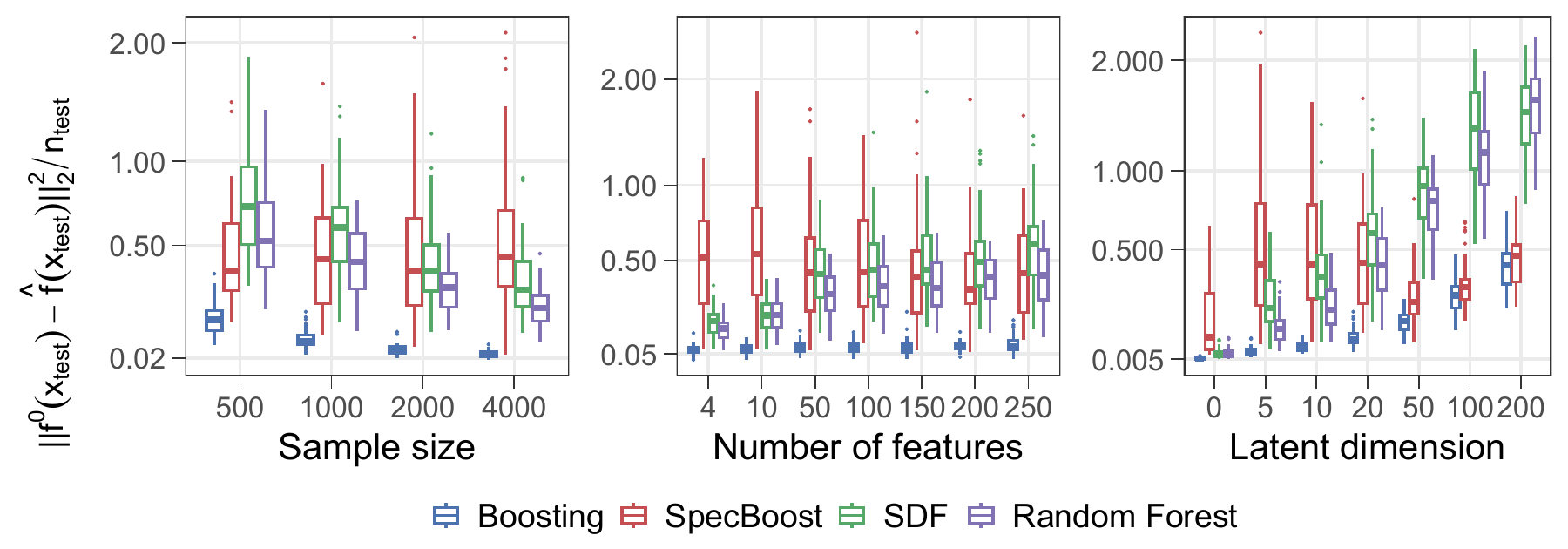}
    \caption{
   Boxplots of the mean squared error for the different methods across simulation scenarios, displayed on a \(\log(1+x)\) scale on the y-axis. 
    }
\end{figure}

\begin{figure}[h!]
    \centering
    \includegraphics[width=0.9\linewidth]{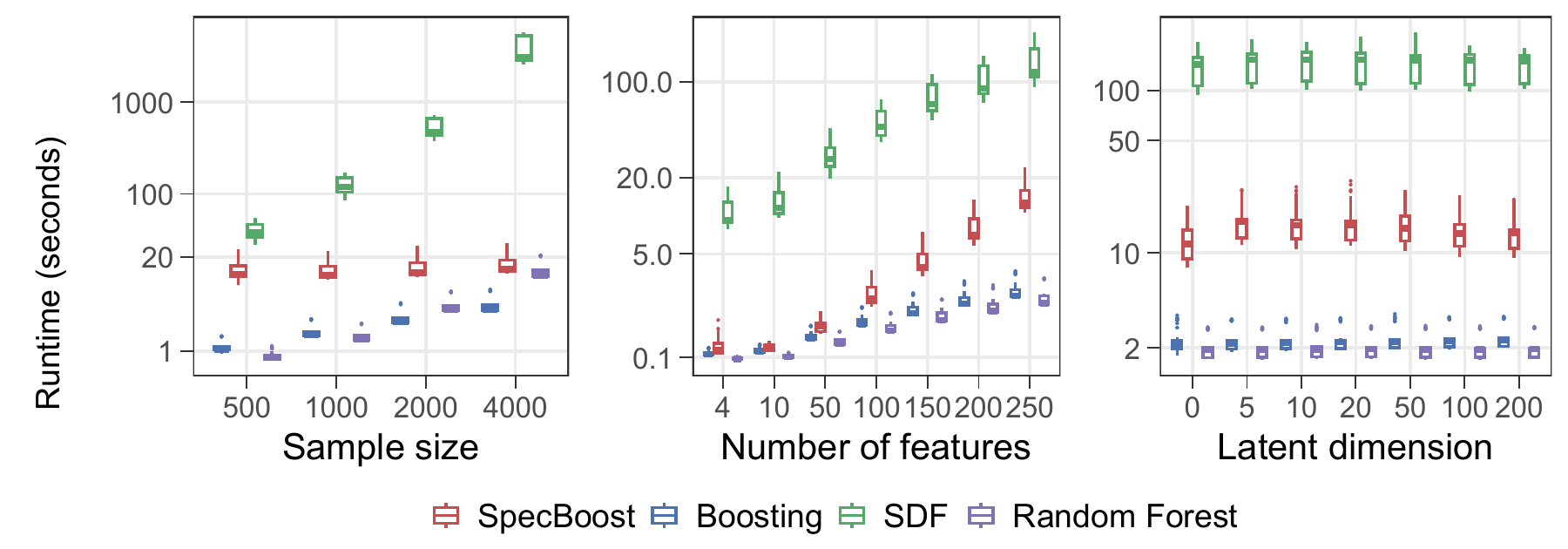}
    \caption{
    Boxplots of the runtime for the different methods across simulation scenarios,  displayed on a \(\log(1+x)\) scale on the y-axis. 
    }
\end{figure}

\begin{figure}[h!]
    \centering
    \includegraphics[width=0.9\linewidth]{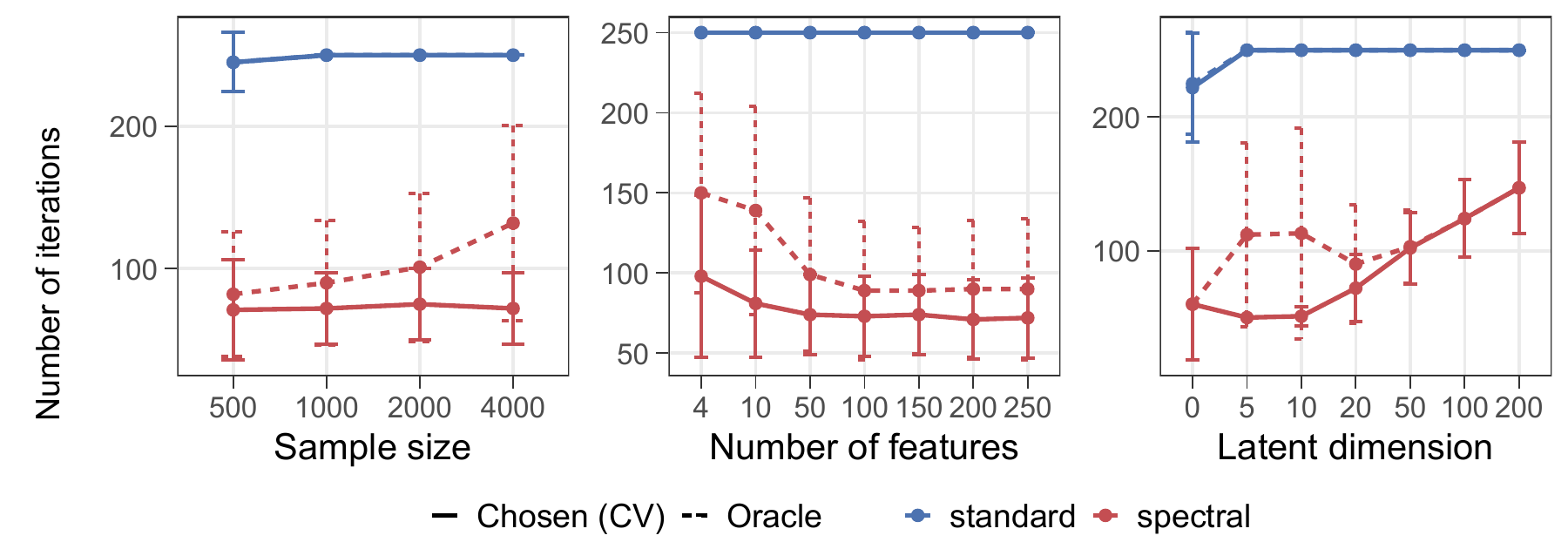}
   \caption{
     Number of boosting iterations selected by the proposed cross-validation rule
    (solid lines) and by the oracle rule with access to the true target function \(f_0\) (dashed lines).
    }
\end{figure}

\begin{figure}[H]
    \centering
    \includegraphics[width=0.9\linewidth]{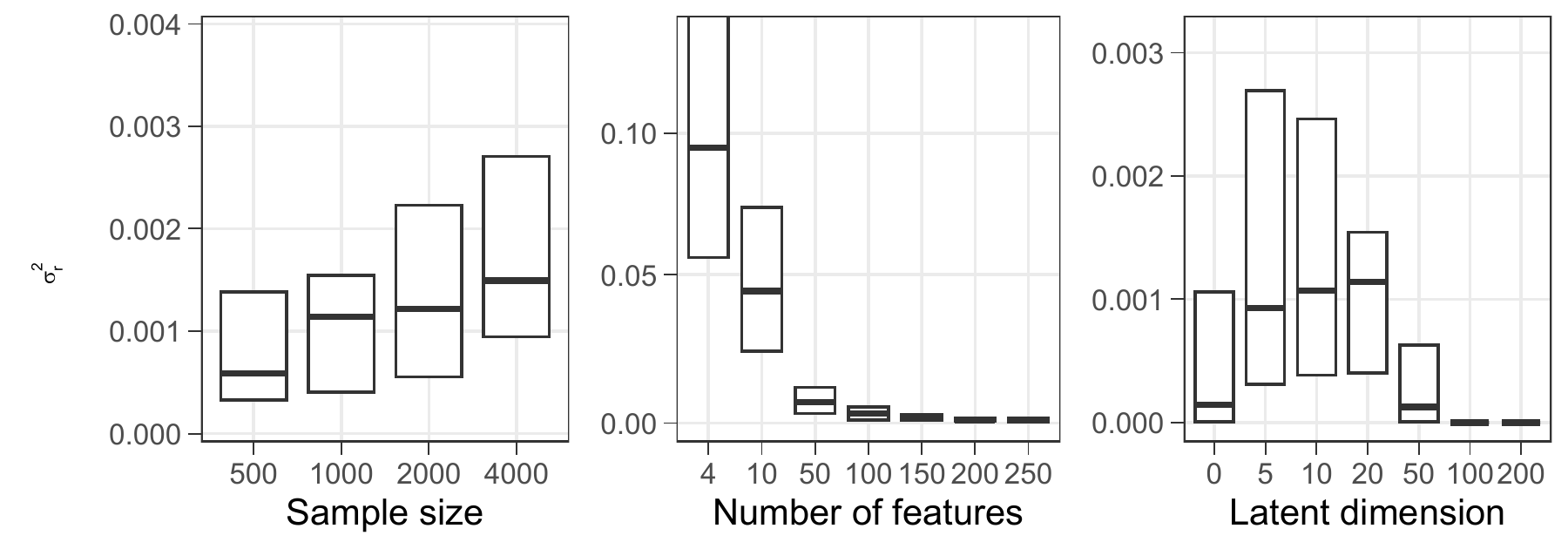}
   \caption{
    Boxplots of the estimated random-effect variance \(\sigma_r^2\),  displayed on a \(\log(1+x)\) scale on the y-axis. Larger values correspond to weaker shrinkage of the random effect.
    }
\end{figure}

\subsection{Classification - No confounding scenario}
In this experiment the confounder's contribution to the outcome is set to zero, i.e. $\delta = 0$.
\begin{figure}[h!]
    \centering
    \includegraphics[width=.9\linewidth]{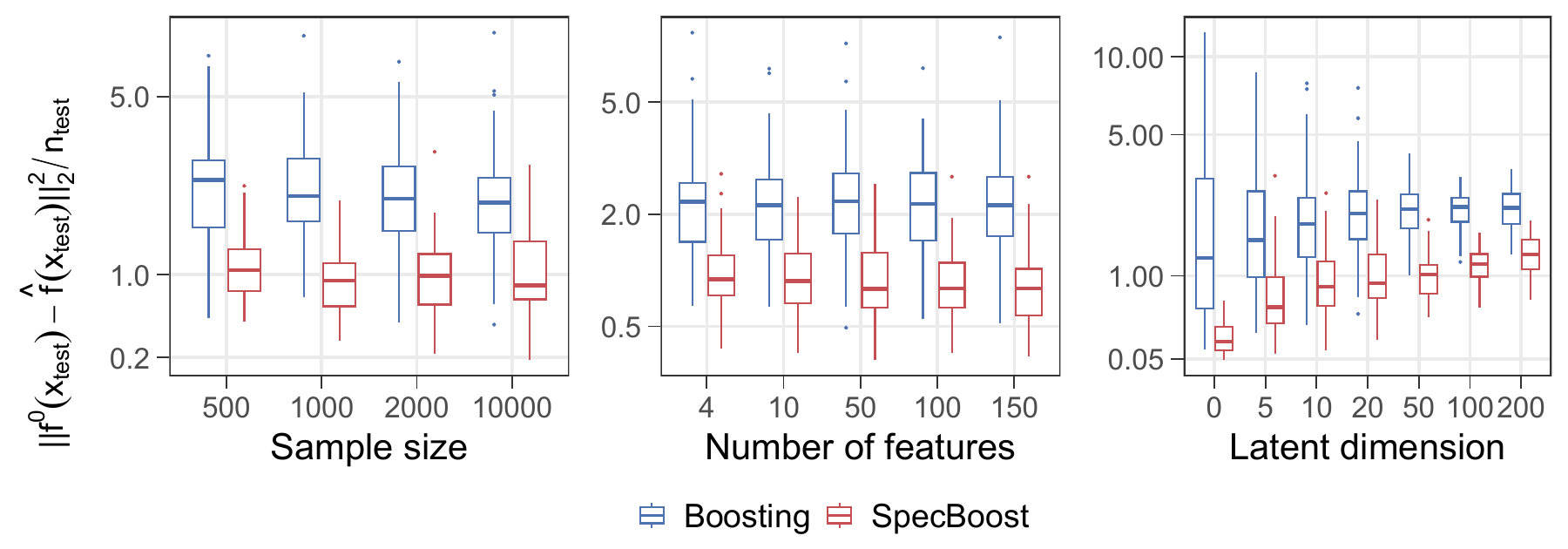}
    \caption{
    Boxplots of the mean squared error for the different methods across simulation scenarios, displayed on a \(\log(1+x)\) scale on the y-axis. 
    }
\end{figure}

\begin{figure}[h]
    \centering
    \includegraphics[width=0.9\linewidth]{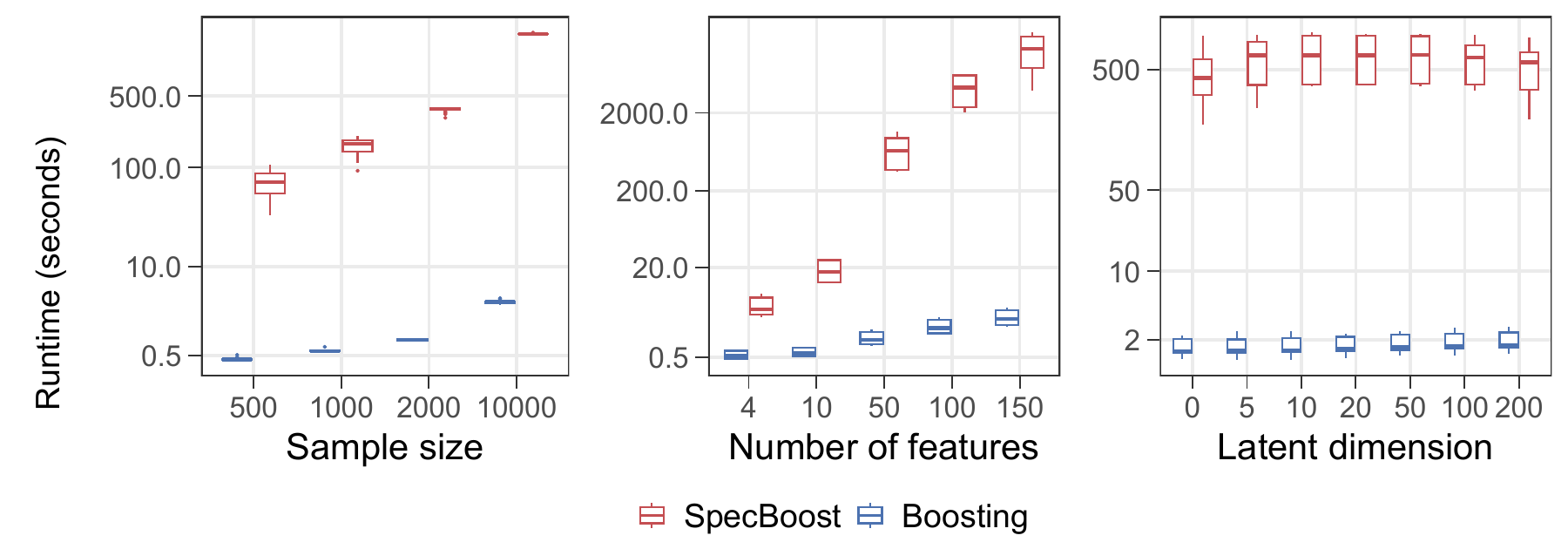}
    \caption{
    Boxplots of the runtime for the different methods across simulation scenarios,  displayed on a \(\log(1+x)\) scale on the y-axis. 
    }
\end{figure}

\begin{figure}[h]
    \centering
    \includegraphics[width=0.9\linewidth]{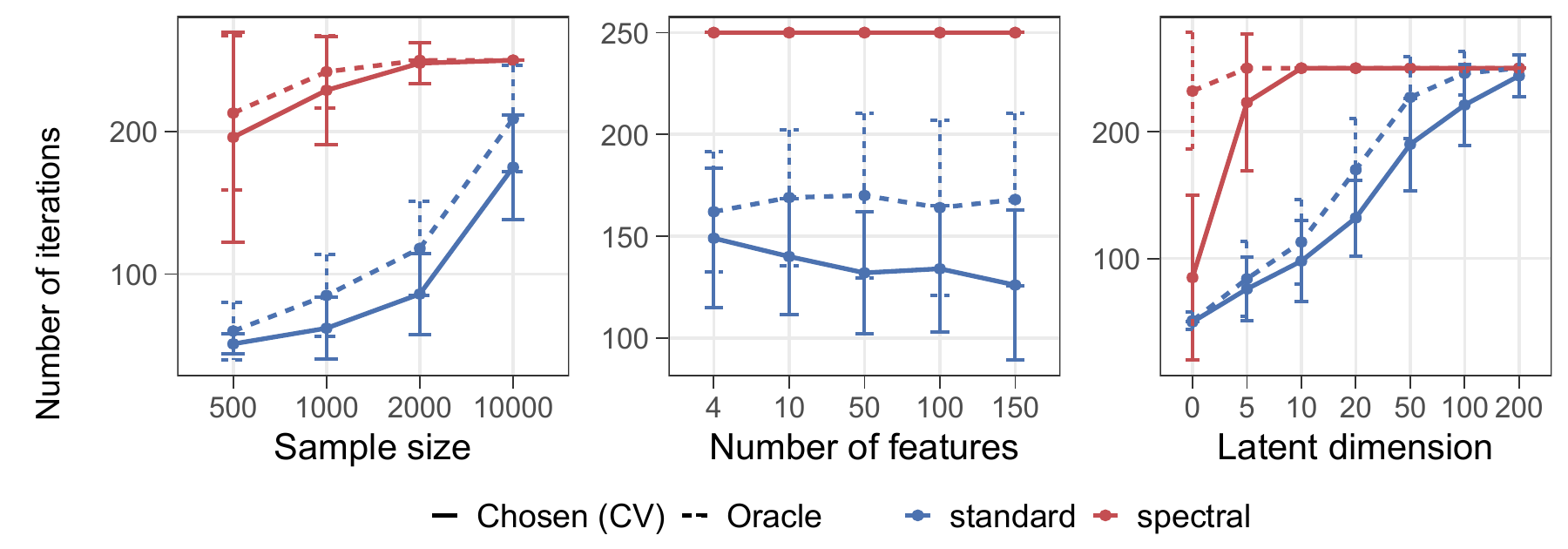}
   \caption{
     Number of boosting iterations selected by the proposed cross-validation rule
    (solid lines) and by the oracle rule with access to the true target function \(f_0\) (dashed lines).
    }
\end{figure}

\begin{figure}[H]
    \centering
    \includegraphics[width=0.9\linewidth]{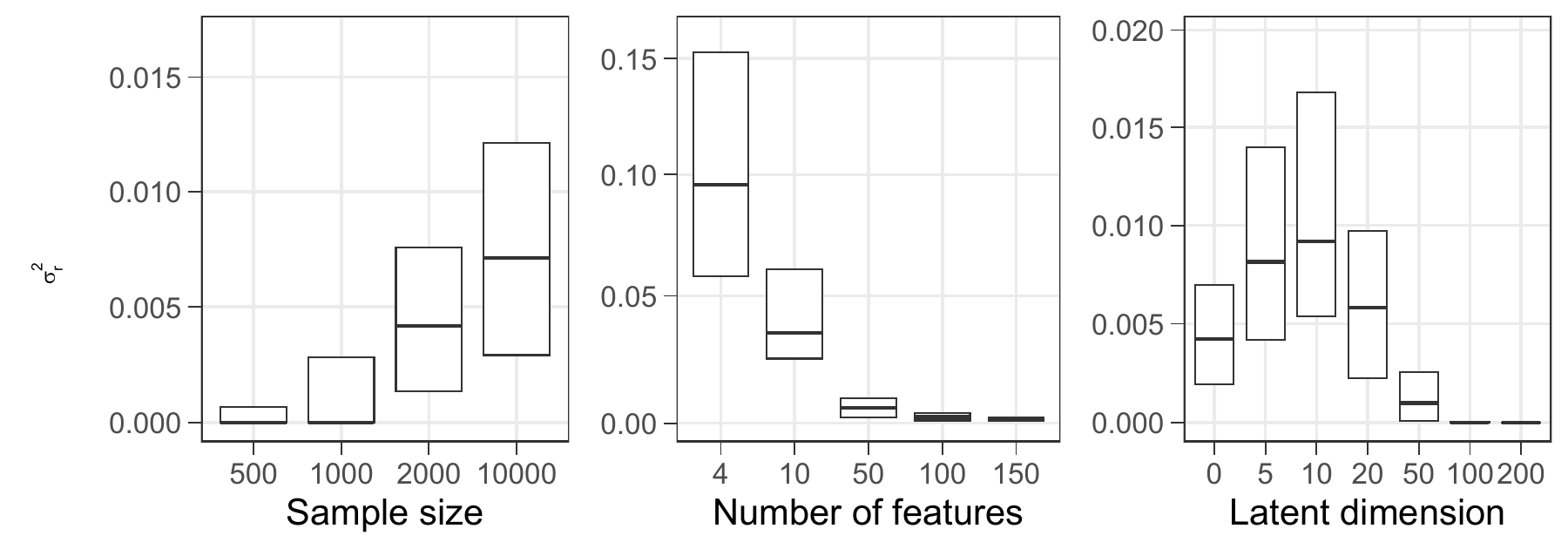}
   \caption{
    Boxplots of the estimated random-effect variance \(\sigma_r^2\),  displayed on a \(\log(1+x)\) scale on the y-axis. Larger values correspond to weaker shrinkage of the random effect.
    }
\end{figure}

\end{document}